\documentclass{article}

\usepackage{times}
\usepackage{graphicx} 
\usepackage{subfigure}

\usepackage{natbib}

\usepackage{algorithm}
\usepackage{algorithmic}

\usepackage{hyperref}

\usepackage{amsfonts}
\usepackage{dashbox}

\usepackage{epsfig}
\usepackage{amsthm}
\usepackage{amsmath}
\usepackage{amssymb}
\usepackage{color}
\usepackage{url}
\usepackage{booktabs}
\usepackage{enumitem}

\usepackage[accepted]{icml2014}

\newcounter{thm_counter}
\newcounter{lem_counter}

\newcounter{cor_counter}

\newtheorem{theorem}[thm_counter]{Theorem}
\newtheorem{lemma}[lem_counter]{Lemma}
\newtheorem{corollary}[cor_counter]{Corollary}

\newcommand{\rc}{\color{black}}
\newcommand{\rt}{\color{black}}
\newcommand{\bt}{\color{black}} 
\newcommand {\lt}{\color{black}} 

\def\supp{\text{supp}}
\def\D{\mathcal{D}}

\icmltitlerunning{Forward-Backward Greedy Algorithms for General Convex Smooth Functions over A Cardinality Constraint}

\begin{document}

\twocolumn[
\icmltitle{Forward-Backward Greedy Algorithms for General Convex Smooth Functions over A Cardinality Constraint}
\icmlauthor{Ji Liu}{ji-liu@cs.wisc.edu}
\icmladdress{Department of Computer Sciences, University of Wisconsin-Madison}
\icmlauthor{Ryohei Fujimaki}{rfujimaki@nec-labs.com}
\icmladdress{Department of Media Analytics, NEC Lab America, Inc.}
\icmlauthor{Jieping Ye}{jieping.ye@asu.edu}
\icmladdress{Department of Computer Science, Arizona State University}

\icmlkeywords{Forward-Backward Greedy Algorithm, Cardinality Constraint}

\vskip 0.3in
]

\begin{abstract}
{\rc
We consider forward-backward greedy algorithms for solving sparse feature selection problems with general convex smooth functions. A state-of-the-art greedy method, the Forward-Backward greedy algorithm (FoBa-obj) requires to solve a large number of optimization problems, thus it is not scalable for large-size problems. The FoBa-gdt algorithm, which uses the gradient information for feature selection at each forward iteration, significantly improves the efficiency of FoBa-obj. In this paper, we systematically analyze the theoretical properties of both forward-backward greedy algorithms. Our main contributions are: 1) We derive better theoretical bounds than existing analyses regarding FoBa-obj for general smooth convex functions; 2) We show that FoBa-gdt achieves the same theoretical performance as FoBa-obj under the same condition: restricted strong convexity condition. Our new bounds are consistent with the bounds of a special case (least squares) and fills a previously existing theoretical gap for general convex smooth functions; 3) We show that the restricted strong convexity condition is satisfied if the number of independent samples is more than $\bar{k}\log d$ where $\bar{k}$ is the sparsity number and $d$ is the dimension of the variable; 4) We apply FoBa-gdt (with the conditional random field objective) to the sensor selection problem for human indoor activity recognition and our results show that FoBa-gdt outperforms other methods (including the ones based on forward greedy selection and L1-regularization).
}
\end{abstract}

\section{Introduction}
Feature selection has been one of the most significant issues in machine learning and data mining. Following the success of Lasso~\citep{Tibshirani94},
learning algorithms with sparse regularization~(a.k.a. sparse
learning) have recently received significant attention. A classical problem is to
estimate a signal $\beta^*\in \mathbb{R}^{d}$ from {\lt a feature
matrix $X\in\mathbb{R}^{n\times d}$ and an observation $y=X\beta^* +
\text{noise} \in \mathbb{R}^{n}$,} under the assumption that
$\beta^*$ is sparse~(i.e., $\beta^*$ has $\bar{k} \ll d$ nonzero
elements). 
Previous studies have proposed many powerful tools to estimate $\beta^*$.
In addition, in certain applications, reducing the
number of features has a significantly practical value~(e.g., sensor selection in our case).


The general sparse learning problems can be formulated as follows \citep{Jalali11}:
{\lt
\begin{equation}
\begin{aligned}
\bar{\beta}:=arg\min_{\beta}:&~Q(\beta; X, y) \quad \quad s.t.:&~\|\beta\|_0\leq \bar{k}.
\label{eq:intro1}
\end{aligned}
\end{equation}
where $Q(\beta; X, y)$ is a convex smooth function in terms of
$\beta$ such as the least square loss~\citep{TroppTIT04}~(regression),
the Gaussian MLE (or log-determinant
divergence)~\citep{Ravikumar11}~(covariance selection), and the logistic
loss~\citep{Kleinbaum10}~(classification). $\|\beta\|_0$
denotes $\ell_0$-norm, that is, the number of nonzero entries of
$\beta \in \mathbb{R}^d$. Hereinafter, we denote $Q(\beta; X, y)$ simply as $Q(\beta)$.}

From an algorithmic viewpoint, we are mainly interested in three
aspects for the estimator $\hat{\beta}$: (i) estimation error $\|\hat{\beta} - \bar{\beta}\|$; (ii)
objective error $Q(\hat{\beta}) - Q(\bar{\beta})$; and (iii) feature
selection error, that is, the difference between
$\rm{supp}(\hat{\beta})$ and $\bar{F}:=\rm{supp}(\bar{\beta})$, where $\rm{supp}(\beta)$ is a feature
index set corresponding to nonzero elements in $\beta$. 
Since the constraint defines a
non-convex feasible region, the problem is non-convex and generally
NP-hard.

There are two types of approaches to solve this problem in the literature.
Convex-relaxation approaches replace $\ell_0$-norm by $\ell_1$-norm
as a sparsity penalty. Such approaches include Lasso
\citep{Tibshirani94}, Danzig selector \citep{Candes05}, and
L1-regularized logistic regression \citep{Kleinbaum10}. Alternative
greedy-optimization approaches include~orthogonal matching pursuit (OMP)~\citep{TroppTIT04, Zhang09}, backward elimination, and forward-backward greedy method~(FoBa)~\citep{Zhang11}, which use
greedy heuristic procedure to estimate sparse signals. Both types
of algorithms have been well studied from both theoretical and empirical
perspectives.

{\rc
FoBa has been shown to give better theoretical properties than LASSO and Dantzig selector for the least squared loss function:} $Q(\beta) = {1\over
2} \|X\beta - y\|^2$~\citep{Zhang11}. \citet{Jalali11} has recently
extended it to general convex smooth functions. Their method and
analysis, however, pose computational and theoretical issues. First,
since FoBa solves a large number of single variable optimization
problems in every forward selection step, it is
computationally expensive for general convex functions if the
sub-problems have no closed form solution. Second, though they have
empirically shown that FoBa performs well for general smooth convex
functions, their theoretical results are weaker than those for
the least square case~\citep{Zhang11}. More precisely, their upper
bound for estimation error is looser and their analysis requires more
{\rc restricted} conditions for {\lt feature selection and signal recovery
consistency.} The question of whether or not FoBa can achieve {\rc the
same theoretical bound in the general case as in the least square case
motivates this work.}

This paper addresses the theoretical and
computational issues associated with the standard FoBa algorithm (hereinafter referred to as FoBa-obj because it solves single variable problems to minimize the objective in each forward selection step). 
We study a new algorithm referred to as ``gradient'' FoBa~(FoBa-gdt) which significantly improves the computational efficiency of FoBa-obj. The key difference is that FoBa-gdt only evaluates gradient information in
individual forward selection steps rather than solving a large number of single variable optimization problems. 
Our contributions are summarized as follows.

{\bf Theoretical Analysis of FoBa-obj and FoBa-gdt}
This paper presents three main theoretical contributions. First, we derive better theoretical bounds for estimation error, objective error, and
feature selection error than existing analyses for FoBa-obj for general smooth convex functions~\citep{Jalali11} under the same condition: restricted strong convexity condition. Second, we show that FoBa-gdt achieves the same theoretical performance as FoBa-obj. Our new bounds are consistent with the bounds of a {\bt special case, i.e., the least square case, and fills in the theoretical gap between the general loss~\citep{Jalali11} and the least squares loss case~\citep{Zhang11}}. Our result also implies an interesting result: when the signal noise ratio is big enough, the NP hard problem \eqref{eq:intro1} can be solved by using FoBa-obj or FoBa-gdt. Third, we show that the restricted strong convexity condition is satisfied for a class of commonly used machine learning objectives, e.g., logistic loss and least square loss, if the number of independent samples is greater than $\bar{k}\log d$ where $\bar{k}$ is the sparsity number and $d$ is the dimension of the variable.

{\bf Application to Sensor Selection}
{\rc
We have applied FoBa-gdt with the CRF loss function (referred to as FoBa-gdt-CRF) to sensor selection from time-series binary location signals~(captured by pyroelectric sensors) for human activity recognition at homes, which is a fundamental problem in smart
home systems and home energy management systems. In comparison with forward greedy and L1-regularized CRFs (referred to as L1-CRF), FoBa-gdt-CRF requires the smallest number of sensors for achieving comparable recognition accuracy. Although this paper mainly focuses on the theoretical analysis for FoBa-obj and FoBa-gdt, we conduct additional experiments to study the behaviors of FoBa-obj and FoBa-gdt in Appendix (part~\ref{sec:addexp}).
}

\subsection{Notation}
Denote $e_j\in \mathbb{R}^d$ as the $j^{th}$ natural basis in the
space $\mathbb{R}^d$. The set difference $A-B$ returns the elements
that are in $A$ but outside of $B$. Given any integer $s>0$, the
restricted strong convexity constants (RSCC) $\rho_-(s)$ and
$\rho_+(s)$ are defined as follows: for any {\bt $\|t\|_0\leq s$ and
$t=\beta'-\beta$}, we require
\begin{align}
  {\rho_-(s)\over 2}\|t\|^2 \leq Q(\beta') - Q(\beta) -\langle \triangledown Q(\beta), t\rangle
  \leq {\rho_+(s)\over 2}\|t\|^2. \label{eqn_RSCC}
\end{align}
Similar definitions can be found in \citep{Bahmani11, Jalali11,
Negahban10, Zhang09}. If the objective function takes the quadratic
form $Q(\beta)={1\over 2}\|X\beta-y\|^2$, then the above definition is
equivalent to the restricted isometric property (RIP)
\citep{CandesTIT05}:
\begin{equation}
{\rho_-(s)}\|t\|^2 \leq \|Xt\|^2 \leq {\rho_+(s)}\|t\|^2, \nonumber
\end{equation}
where the well known RIP constant can be defined as $\delta =
\max\{1-\rho_-(s),~\rho_+(s)-1\}$. To give tighter values for $\rho_+(.)$ and $\rho_-(.)$, we only require \eqref{eqn_RSCC} to hold for all $\beta\in\D_s:=\{\|\beta\|_0\leq s~|~Q(\beta)\leq Q(0)\}$ throughout this paper. Finally we define $\hat\beta(F)$ as $ \hat\beta(F):=\arg\min_{\rm{supp}(\beta)\subset
F}:Q(\beta).$ Note that the problem is convex as long as $Q(\beta)$
is a convex function. Denote $\bar{F}:=\rm{\supp}(\bar{\beta})$ and $\bar{k}:=|\bar{F}|$.

We make use of order notation throughout this paper. If $a$ and $b$ are both positive quantities that depend on {\rc $n$ or $p$}, we write $a=O(b)$ if $a$ can be bounded by a fixed multiple of $b$ for all sufficiently large dimensions. We write $a=o(b)$ if for {\rm any} positive constant $\phi>0$, we have $a\leq \phi b$ for all sufficiently large dimensions. We write $a=\Omega(b)$ if both $a=O(b)$ and $b=O(a)$ hold.

\begin{algorithm}[h!]
\caption{FoBa (\fbox{FoBa-obj}~\dbox{FoBa-gdt})}
\begin{algorithmic}[1]
\REQUIRE \fbox{$\delta > 0$}~\dashbox{$\epsilon > 0$}
\ENSURE $\beta^{(k)}$
\STATE Let $F^{(0)}=\emptyset$, $\beta^{(0)}=0$, $k=0$, 
\WHILE{TRUE} {\label{alg_obj_stop_start}} \STATE \%\% stopping
determination \IF{\fbox{$Q(\beta^{(k)})-\min_{\alpha,j\notin
F^{(k)}}Q(\beta^{(k)}+\alpha e_j)< \delta$}\par\dbox{$\|\triangledown
Q(\beta^{(k)})\|_\infty< \epsilon$}} \STATE break \ENDIF
{\label{alg_obj_stop_end}} \STATE \%\% forward step
{\label{alg_obj_forward_start}}
\STATE 
\fbox{$i^{(k)}=\arg\min_{i\notin F^{(k)}}\{\min_\alpha
Q(\beta^{^{(k)}}+\alpha e_i)\}$}\par\dbox{$i^{(k)}=\arg\max_{i\notin F^{(k)}}:~|\nabla Q(\beta^{(k)})_i|$} \STATE $F^{(k+1)}=F^{(k)}\cup
\{i^{(k)}\}$ \STATE $\beta^{(k+1)} = \hat{\beta}(F^{(k+1)})$ \STATE
$\delta^{(k+1)}=Q(\beta^{(k)})-Q(\beta^{(k+1)})$ \STATE $k=k+1$
\label{alg_obj_forward_end} \STATE \%\% backward step
\label{alg_obj_backward_start} \WHILE{TRUE}
\IF{$\min_{i\in F^{(k+1)}}Q(\beta^{(k)}-\beta_i^{(k)}e_i)-Q(\beta^{(k)}) \geq
\delta^{(k)}/2$} \STATE break \ENDIF
\STATE 
$i^{(k)}=\arg\min_i Q(\beta^{(k)}-\beta^{(k)}_ie_i)$ \STATE $k=k-1$
\STATE $F^{(k)}=F^{(k+1)}-\{i^{(k+1)}\}$ \STATE $\beta^{(k)} =
\hat{\beta}(F^{(k)})$ \ENDWHILE \label{alg_obj_backward_end}
\ENDWHILE
\end{algorithmic}
\label{alg_FoBaobj}
\end{algorithm}

\section{Related Work}
\citet{TroppTIT04} investigated the behavior of the orthogonal
matching pursuit (OMP) algorithm for the least square case, and
proposed a sufficient condition (an $\ell_\infty$ type condition)
for guaranteed feature selection consistency. \citet{Zhang09}
generalized this analysis to {\rt the case of measurement noise}. In
statistics, OMP is known as \rm{boosting} \citep{Buhlmann06} and
similar ideas have been explored in Bayesian network learning
\citep{Chickering02}. \citet{Shalev-ShwartzSZ10}
extended OMP to the general convex smooth function and studied
the relationship between objective value reduction and output
sparsity. Other greedy methods such as ROMP \citep{Needell09} and
CoSaMp \citep{Needell08} were studied and shown to have theoretical
properties similar to those of OMP. \citet{Zhang11} proposed a Forward-backward (FoBa)
greedy algorithm for the least square case, which is an extension of OMP but has stronger theoretical guarantees as well as better empirical performance: feature selection consistency is guaranteed under the sparse eigenvalue condition, which is an $\ell_2$ type condition weaker than the $\ell_\infty$ type condition. Note that if the data matrix is a Gaussian random matrix, the $\ell_2$ type condition requires the measurements $n$ to
be of the order of $O(s\log d)$ where $s$ is the sparsity of the true
solution and $d$ is the number of features, while the
$\ell_\infty$ type condition requires $n=O(s^2\log d)$; see~\citep{zhangzhang12, LiuJMLR12}. \citet{Jalali11} and \citet{JohnsonJR12} extended the FoBa algorithm to general convex functions and applied it to
sparse inverse covariance estimation problems.

Convex methods, such as LASSO \citep{ZhaoY06} and Dantzig selector
\citep{Candes05}, were proposed for sparse learning. The basic idea
behind these methods is to use the $\ell_1$-norm to approximate the
$\ell_0$-norm in order to transform problem~\eqref{eq:intro1} into a
convex optimization problem. They usually require restricted
conditions referred to as irrepresentable conditions (stronger than
the RIP condition) for guaranteed feature selection consistency
\citep{Zhang11}. A multi-stage procedure on LASSO and Dantzig
selector \citep{LiuJMLR12} relaxes such condition, but it is
still stronger than RIP.


\section{The Gradient FoBa Algorithm} \label{sec:alg}
{\rc
This section introduces the standard FoBa algorithm, that is, FoBa-obj, and its variant FoBa-gdt. Both algorithms start from an empty feature pool $F$ and follow the same procedure in every iteration consisting of two steps: a forward step and a backward step. The forward step evaluates the ``goodness'' of all features outside of the current feature set $F$, selects the best feature to add to the current feature pool $F$, and then optimizes the corresponding coefficients of all features in the current feature pool $F$ to obtain a new $\beta$. The elements of $\beta$ in $F$ are nonzero and the rest are zeros. The backward step \emph{iteratively} evaluates the ``badness'' of all features outside of the current feature set $F$, removes ``bad'' features from the current feature pool $F$, and recomputes the optimal $\beta$ over the current feature set $F$. Both algorithms use the same definition of ``badness'' for a feature: the increment of the objective after removing this feature. Specifically, for any features $i$ in the current feature pool $F$, the ``badness'' is defined as $Q(\beta-\beta_i e_i) - Q(\beta)$, which is a positive number. It is worth to note that the forward step selects one and only one feature while the backward step may remove zero, one, or more features. Finally, both algorithms terminate when no ``good'' feature can be identified in the forward step, that is, the ``goodness'' of all features outside of $F$ is smaller than a threshold. 

The main difference between FoBa-obj and FoBa-gdt lies in the definition of ``goodness'' in the forward step and their respective stopping criterion. FoBa-obj evaluates the goodness of a feature by its maximal reduction of the objective function. Specifically, the ``goodness'' of feature $i$ is defined as $Q(\beta) - \min_{\alpha}Q(\beta+\alpha e_i)$ (a larger value indicates a better feature). This is a direct way to evaluate the ``goodness'' since our goal is to decrease the objective as much as possible under the cardinality condition. However, it may be computationally expensive since it requires solving a large number of one-dimensional optimization problems, which may or may not be solved in a closed form. To improve computational efficiency in such situations, FoBa-gdt uses the partial derivative of $Q$ with respect to individual coordinates (features) as its ``goodness' measure: specifically, the ``goodness'' of feature $i$ is defined as $|\nabla Q(\beta)_i|$. Note that the two measures of ``goodness'' are always nonnegative. If feature $i$ is already in the current feature set $F$, its ``goodness'' score is always zero, no matter which measure to use. We summarize the details of FoBa-obj and FoBa-gdt in Algorithm~\ref{alg_FoBaobj}: the plain texts correspond to the common part of both algorithms, and the ones with solid boxes and dash boxes correspond to their individual parts. The superscript $(k)$ denotes the $k^{th}$ iteration incremented/decremented in the forward/backward steps.

Gradient-based feature selection has been used in a forward greedy method \citep{Zhang11a}. FoBa-gdt extends it to a Forward-backward procedure (we present a detailed theoretical analysis of it in the next section). The main workload in the forward step for FoBa-obj is on Step 4, whose complexity is $O(TD)$, where $T$ represents the iterations needed to solve $\min_{\alpha}:~Q(\beta^{(k)}+\alpha e_j)$ and $D$ is the number of features outside of the current feature pool set $F^{(k)}$. In comparison, the complexity of Step 4 in FoBa is just $O(D)$. 
When $T$ is large, we expect FoBa-gdt to be much more computationally efficient. The backward steps of both algorithms are identical. The computational costs of the backward step and the forward step are comparable in FoBa-gdt (but not FoBa-obj), because their main work loads are on Step 10 and Step 21 (both are solving $\hat{\beta}(.)$) respectively and the times of running Step 21 is always less than that of Step 10. 

}

\section{Theoretical Analysis} \label{sec:result}
This section first gives the termination condition of Algorithms~\ref{alg_FoBaobj} with FoBa-obj and FoBa-gdt because the number of iterations directly affect the values of RSCC ($\rho_+(.)$, $\rho_-(.)$, and their ratio), which are the key factors in our main results. Then we discuss the values of RSCC in a class of commonly used machine learning objectives. Next we present the main results of this paper, including upper bounds on objective, estimation, and feature selection errors for both FoBa-obj and FoBa-gdt. We compare our results to those of existing analyses of FoBa-obj and show that our results fill the theoretical gap between the least square loss case and the general case.

\subsection{Upper Bounds on Objective, Estimation, and Feature Selection Errors} \label{sec:result:1}


{\rc
We first study the termination conditions of FoBa-obj and FoBa-gdt, as summarized in Theorems~\ref{thm_main2_obj} and \ref{thm_main2} respectively.
}
\begin{theorem}\label{thm_main2_obj}
Take $\delta > {4\rho_+(1)\over
\rho_-(s)^2}\|\triangledown Q(\bar\beta)\|_\infty^2$ in Algorithm~\ref{alg_FoBaobj} with FoBa-obj where $s$ can be any positive integer satisfying
$s\leq n$ and
  \begin{equation}
  (s-\bar{k})>(\bar{k}+1)\left[\left(\sqrt{\rho_+(s)\over \rho_-(s)}+1\right){2\rho_+(1)\over \rho_-(s)}\right]^2.
    \label{eq:thm_1}
  \end{equation}
Then the algorithm terminates at some $k\leq s-\bar{k}$.
\end{theorem}
\begin{theorem}\label{thm_main2}
Take $\epsilon > {2\sqrt{2}\rho_+(1)\over
\rho_-(s)}\|\triangledown Q(\bar\beta)\|_\infty$ in Algorithm~\ref{alg_FoBaobj} with FoBa-gdt, where $s$ can be any positive integer satisfying $s\leq n$ and Eq.(\ref{eq:thm_1}).
Then the algorithm terminates at some $k\leq s-\bar{k}$.
\end{theorem}
{\rc
To simply the results, we first assume that the condition number $\kappa(s):= {\rho_+(s)/\rho_-(s)}$ is bounded (so is ${\rho_+(1) / \rho_-(s)}$ because of $\rho_+(s)\geq \rho_+(1)$). Then both FoBa-obj and FoBa-gdt terminate at some $k$ proportional
to the sparsity $\bar{k}$, {\rt similar to OMP
\citep{Zhang11a} and FoBa-obj \citep{Jalali11, Zhang11}.} Note that the value of $k$ in Algorithm~\ref{alg_FoBaobj} is exactly the cardinality of $F^{(k)}$ and the sparsity of $\beta^{(k)}$. Therefore, Theorems~\ref{thm_main2_obj} and \ref{thm_main2} imply that if $\kappa(s)$ is bounded, FoBa-obj and FoBa-gdt will output a solution with sparsity proportional to that of the true solution $\bar{\beta}$.
}

Most existing works simply assume that $\kappa(s)$ is bounded or have similar assumptions. We make our analysis more complete by discussing the values of $\rho_+(s)$, $\rho_-(s)$, and their ratio $\kappa(s)$. Apparently, if $Q(\beta)$ is strongly convex and Lipschitzian, then $\rho_-(s)$ is bounded from below and $\rho_+(s)$ is bounded from above, thus restricting the ratio $\kappa(s)$. To see that $\rho^+(s)$, $\rho^-(s)$, and $\kappa(s)$ may still be bounded under milder conditions, we consider a common structure for $Q(\beta)$ used in many machine learning formulations:
\begin{equation}
Q(\beta) = {1\over n}\sum_{i=1}^n l_i(X_{i.}\beta, y_i) + R(\beta)
\end{equation}
where $(X_{i.}, y_i)$ is the $i^{th}$ training sample with $X_{i.}\in \mathbb{R}^d$ and $y_i\in \mathbb{R}$, $l_i(.,.)$ is convex with respect to the first argument and could be different for different $i$, and both $l_i(.,.)$ and $R(.)$ are twice differentiable functions. $l_i(.,.)$ is typically the loss function, e.g., the quadratic loss $l_i(u, v) = (u-v)^2$ in regression problems and the logistic loss $l_i(u, v) = \log (1+\exp\{-uv\})$ in classification problems. $R(\beta)$ is typically the regularization, e.g., $R(\beta)={\mu\over 2}\|\beta\|^2$.

\begin{theorem} \label{thm_kappa}
Let $s$ be a positive integer less than $n$, and $\lambda^-$, $\lambda^+$, $\lambda^-_R$, and $\lambda^+_R$ be positive numbers satisfying
\[
\lambda^- \leq \nabla_1^2l_i(X_{i.}\beta, y_i) \leq \lambda^+,\quad \lambda^-_RI\preceq\nabla^2 R(\beta) \preceq \lambda^+_RI
\]
($\nabla_1^2l_i(.,.)$ is the second derivative with respect to the first argument) for any $i$ and $\beta\in\D_s$. Assume that $\lambda^-_R+0.5\lambda^- > 0$ and the sample matrix $X\in\mathbb{R}^{n\times d}$ has independent sub-Gaussian isotropic random rows or columns (in the case of columns, all columns should also satisfy $\|X_{.j}\|=\sqrt{n}$). 
If the number of samples satisfies $n\ge C s\log d$, then
\begin{subequations}
\begin{align}
\rho_+(s) \leq & \lambda^+_R + 1.5\lambda^+ \label{eqn_thm_kappa_+}\\
\rho_-(s) \geq & \lambda^-_R + 0.5\lambda^- \label{eqn_thm_kappa_-}\\
\kappa(s) \leq & \frac{\lambda^+_R + 1.5\lambda^+}{\lambda^-_R + 0.5\lambda^-}=: \kappa \label{eqn_thm_kappa_+-}
\end{align} \label{eqn_thm_kappa_all}
\end{subequations}
hold with high probability\footnote{``With high probability'' means that the probability converges to $1$ with the problem size approaching to infinity.}, where $C$ is a fixed constant. Furthermore, define $\bar{k}$, $\bar{\beta}$, and $\delta$ (or $\epsilon$) in Algorithm~\ref{alg_FoBaobj} with FoBa-obj (or FoBa-gdt) as in Theorem~\ref{thm_main2_obj} (or Theorem~\ref{thm_main2}). Let
\begin{equation}
s = \bar{k} + 4\kappa^2(\sqrt{\kappa}+1)^2(\bar{k}+1)
\label{eqn_thm_kappa_s}
\end{equation}
and $n\geq Cs\log d$. We have that $s$ satisfies \eqref{eq:thm_1} and Algorithm~\ref{alg_FoBaobj} with FoBa-obj (or FoBa-gdt) terminates within at most $4\kappa^2(\sqrt{\kappa}+1)^2(\bar{k}+1)$ iterations with high probability.
\end{theorem}
Roughly speaking, if the number of training samples is large enough, i.e., $n\geq \Omega(\bar{k}\log d)$ (actually it could be much smaller than the dimension $d$ of the train data), we have the following with high probability: Algorithm~\ref{alg_FoBaobj} with FoBa-obj or FoBa-gdt outputs a solution with sparsity at most $\Omega(\bar{k})$ (this result will be improved when the nonzero elements of $\bar{\beta}$ are strong enough, as shown in Theorems~\ref{thm_main1_obj} and \ref{thm_main1}); $s$ is bounded by $\Omega(\bar{k})$; and $\rho_+(s)$, $\rho_-(s)$, and $\kappa(s)$ are bounded by constants. One important assumption is that the sample matrix $X$ has independent sub-Gaussian isotropic random rows or columns. In fact, this assumption is satisfied by many natural examples, including Gaussian and Bernoulli matrices, general bounded random matrices whose entries are independent bounded random variables with zero mean and unit variances. Note that from the definition of ``sub-Gaussian isotropic random vectors'' \citep[Definitions 19 and 22]{Vershynin11}, it even allows the dependence within rows or columns but not both. Another important assumption is $\lambda^-_R+0.5\lambda^- > 0$, which means that either $\lambda^-_R$ or $\lambda^-$ is positive (both of them are nonnegative from the convexity assumption). We can simply verify that (i) for the quadratic case $Q(\beta)={1\over n}\sum_{i=1}^n(X_{i.}\beta-y_i)^2$, we have $\lambda^-=1$ and $\lambda^-_R=0$; (ii) for the logistic case with bounded data matrix $X$, that is $Q(\beta)={1\over n}\sum_{i=1}^n\log (1+\exp\{-X_{i.}\beta y_i\}) + {\mu\over 2} \|\beta\|^2$, we have $\lambda^-_R = \mu>0$ and $\lambda^- > 0$ because $\D_s$ is bounded in this case.

Now we are ready to present the main results: the upper bounds of estimation error, objective
error, and feature selection error for both algorithms. $\rho_+(s)$, $\rho_+(1)$, and $\rho_-(s)$ are involved in all bounds below. One can simply treat them as constants in understanding the following results, since we are mainly interested in the scenario when the number of training samples is large enough. We omit proofs due to space limitations (the proofs are provided in Appendix). The main results for FoBa-obj and FoBa-gdt are presented in Theorems~\ref{thm_main1_obj} and \ref{thm_main1} respectively. 

\begin{theorem}\label{thm_main1_obj}
Let $s$ be any number that satisfies \eqref{eq:thm_1} and choose $\delta$ as in Theorem~\ref{thm_main2_obj} for Algorithm~\ref{alg_FoBaobj} with FoBa-obj. Consider the output $\beta^{(k)}$ and its support set $F^{(k)}$. We have
\begin{align*}
\|\beta^{(k)}-\bar\beta\|^2 \leq& {16\rho_+^2(1)\delta\over \rho_-^2(s)}\bar\Delta,\\
Q(\beta^{(k)})-Q(\bar\beta)\leq& {2\rho_+(1)\delta\over \rho_-(s)}\bar\Delta,\\
{\rho_-(s)^2\over 8\rho_+(1)^2}|F^{(k)}-\bar{F}|\leq&
|\bar{F}-F^{(k)}|\leq 2\bar\Delta,
\end{align*}
where $\gamma={4\sqrt{\rho_+(1)\delta}\over \rho_-(s)}$ and
$\bar\Delta := |\{j\in \bar{F}-F^{(k)}: |\bar{\beta}_j|<\gamma\}|$.
\end{theorem}

\begin{theorem}\label{thm_main1}
Let $s$ be any number that satisfies \eqref{eq:thm_1} and choose $\epsilon$ as in Theorem~\ref{thm_main2} for Algorithm~\ref{alg_FoBaobj} with FoBa-gdt. Consider the output $\beta^{(k)}$ and its support
set $F^{(k)}$. We have
\begin{align*}
\|\beta^{(k)}-\bar\beta\|^2 \leq& {8\epsilon^2\over \rho_-^2(s)}\bar\Delta,\\
Q(\beta^{(k)})-Q(\bar\beta)\leq& {\epsilon^2\over \rho_-(s)}\bar\Delta,\\
{\rho_-(s)^2\over 8\rho_+(1)^2}|F^{(k)}-\bar{F}|\leq&
|\bar{F}-F^{(k)}|\leq 2\bar\Delta,
\end{align*}
where $\gamma={2\sqrt{2}\epsilon\over \rho_-(s)}$ and $\bar\Delta :=
|\{j\in \bar{F}-F^{(k)}: |\bar{\beta}_j|<\gamma\}|$.
\end{theorem}
Although FoBa-obj and FoBa-gdt use different criteria to evaluate the
``goodness'' of each feature, they actually guarantee the same
properties. Choose $\epsilon^2$ and $\delta$ in the order of $\Omega(\|\nabla Q(\bar{\beta})\|^2_{\infty})$. For both algorithms, we have that the estimation error $\|\beta^{(k)}-\bar{\beta}\|^2$ and the objective error $Q(\beta^{(k)})-Q(\bar\beta)$ are bounded by $\Omega(\bar\Delta \|\nabla Q(\bar{\beta})\|^2_{\infty})$, and the feature selection errors $|F^{(k)}-\bar{F}|$ and $|\bar{F}-F^{(k)}|$ are bounded by $\Omega(\bar\Delta)$. $\|\nabla Q(\bar\beta)\|_\infty$ and $\bar\Delta$ are two key factors in these bounds. $\|\nabla Q(\bar\beta)\|_\infty$ roughly represents the noise level\footnote{To see this, we can consider the least square case (with standard noise assumption and each column of the measurement matrix $X\in \mathbb{R}^{n\times d}$ is normalized to 1): $\|\nabla Q(\bar\beta)\|_\infty \leq \Omega(\sqrt{n^{-1}\log d}\sigma)$ holds with high probability, where $\sigma$ is the standard derivation.}. $\bar{\Delta}$ defines the number of weak channels of the true solution $\bar\beta$ in $\bar{F}$. One can see that if all channels of $\bar{\beta}$ on $\bar{F}$ are strong enough, that is, $|\bar{\beta}_j| > \Omega(\|\nabla Q(\bar{\beta})\|_\infty)~\forall j\in\bar{F}$, $\bar\Delta$ turns out to be $0$. In other words, all errors (estimation error, objective error, and feature selection error) become $0$, when the signal noise ratio is big enough. Note that under this condition, the original NP hard problem~\eqref{eq:intro1} is solved exactly, which is summarized in the following corollary:
\begin{corollary}
Let $s$ be any number that satisfies \eqref{eq:thm_1} and choose $\delta$ (or $\epsilon$) as in Theorem~\ref{thm_main2_obj} (or \ref{thm_main2}) for Algorithm~\ref{alg_FoBaobj} with FoBa-gdt (or FoBa-obj). If $${|\bar{\beta}_j| \over \|\nabla Q(\bar{\beta})\|_{\infty}} \ge {8\rho_+(1)\over \rho_-^2(s)}\quad \forall j\in \bar{F},$$ then problem~\eqref{eq:intro1} can be solved exactly.
\end{corollary}
One may argue that since it is difficult to set $\delta$ or $\epsilon$, it is still hard to solve \eqref{eq:intro1}. In practice, one does not have to set $\delta$ or $\epsilon$ and only needs to run Algorithm~\ref{alg_FoBaobj} without checking the stopping condition until all features are selected. Then the most recent $\beta^{(\bar{k})}$ gives the solution to \eqref{eq:intro1}.

\subsection{Comparison for the General Convex Case}
\citet{Jalali11} analyzed FoBa-obj for general convex smooth functions and here we compare our results to theirs. They chose the true model 
$\beta^*$ as the target rather than the true solution $\bar\beta$.
In order to simplify the comparison, we assume that the distance
between the true solution and the true model is not too
great\footnote{This assumption is not absolutely fair, but holds in
many cases, such as in the least square case, which will be made clear in
Section~\ref{sec:result:3}.}, that is, we have $\beta^*\approx
\bar\beta$, $\rm{supp}(\beta^*)=\rm{supp}(\bar{\beta})$, and $\|\nabla Q(\beta^*)\|_\infty \approx \|\nabla
Q(\bar\beta)\|_\infty$. We compare our results from
Section~\ref{sec:result:1} and the results in \citep{Jalali11}. In
the estimation error comparison, we have from our results:
\begin{align*}
&\|\beta^{(k)}-\beta^*\|\approx \|\beta^{(k)}-\bar\beta\| \\
\leq & \Omega({\bar\Delta}^{1/2}\|\nabla Q(\bar\beta)\|_\infty)
\approx  \Omega({\bar\Delta}^{1/2}\|\nabla Q(\beta^*)\|_\infty)
\end{align*}
and from the results in \citep{Jalali11}: $\|\beta^{(k)}-\beta^*\| \leq \Omega(\bar{k}\|\nabla Q(\beta^*)\|_\infty).$
Note that ${\Delta}^{1/2} \leq \bar{k}^{1/2} \ll \bar{k}$.
Therefore, {\rt under our assumptions with respect to $\beta^*$} and $\bar\beta$,
our analysis gives a tighter bound. Notably, when there are a large number of
strong channels in $\bar\beta$ (or approximately $\beta^*$), we will have
$\bar{\Delta} \ll \bar{k}$.

Let us next consider the condition {\rt required for feature selection
consistency, that is,
$\rm{supp}(F^{(k)})=\rm{supp}(\bar{F})=\rm{supp}(\beta^*)$. We have from our results:
$$\|\bar{\beta}_j\| \geq \Omega(\|\nabla Q(\bar\beta)\|_\infty)~\forall j \in \rm{supp}(\beta^*)$$and from the results in \citep{Jalali11}:
$$\|{\beta}^*_j\| \geq \Omega(\bar{k}\|\nabla Q(\beta^*)\|_\infty)~\forall j \in \rm{supp}(\beta^*).$$
When $\beta^*\approx \bar\beta$ and $\|\nabla Q(\beta^*)\|_\infty
\approx \|\nabla Q(\bar\beta)\|_\infty$, our results guarantee feature selection consistency under a weaker condition.}

\subsection{A Special Case: Least Square Loss} \label{sec:result:3}
We next consider the least square case: $Q(\beta)={1\over
2}\|X\beta-y\|^2$ and shows that our analysis for the two algorithms in
Section~\ref{sec:result:1} fills in a theoretical gap between this
special case and the general convex smooth case.

{\rt Following previous studies \citep{Candes05, Zhang11a, ZhaoY06}}, we
assume that $y=X\beta^* +\varepsilon$ where the entries in
$\varepsilon$ are independent random sub-gaussian variables,
$\beta^*$ is the true model with the support set $\bar{F}$ and
the sparsity number $\bar{k}:=|\bar{F}|$, and $X\in
\mathbb{R}^{n\times d}$ is normalized as $\|X_{.i}\|^2=1$ for all
columns $i=1,\cdots,d$. We then have following inequalities with
high probability \citep{Zhang09}:
\begin{align}
&\|\nabla Q(\beta^*)\|_\infty = \|X^T\varepsilon\|_\infty \leq \Omega(\sqrt{n^{-1}\log d}), \label{eqn_gQstar}
\\
&\|\nabla Q(\bar\beta)\|_\infty \leq \Omega(\sqrt{n^{-1}\log d}), \label{eqn_gQbar}
\\
&\|\bar\beta - \beta^*\|_2 \leq \Omega(\sqrt{n^{-1}\bar{k}}), \label{eq:secResult1}
\\
&\|\bar\beta-\beta^*\|_\infty \leq \Omega(\sqrt{n^{-1}\log \bar{k}}), \label{eq:secResult2}
\end{align}
implying that $\bar\beta$ and $\beta^*$ are quite close when
the true model is really sparse, that is, when $\bar{k}\ll n$.

An analysis for FoBa-obj in the least square case \citep{Zhang11} has
indicated that the following estimation error bound holds with high probability:
\begin{align}
&\|\beta^{(k)}-\beta^*\|^2 \leq \Omega(n^{-1}(\bar{k}+ \nonumber
\\
\quad &\log d|\{j\in \bar{F}:~|\beta_j^*|\leq \Omega(\sqrt{n^{-1}\log d})\}|))
\label{eq:secIntro1}
\end{align}
as well as the following condition for feature selection consistency: 
if $|{\beta}_j^*|\geq \Omega(\sqrt{n^{-1}\log d})~\forall j\in \bar{F}$, then
\begin{align}
{\rm supp}(\beta^{(k)})={\rm supp}(\beta^*)
\label{eq:secIntro2}
\end{align}
Applying the analysis for general convex smooth
cases in \citep{Jalali11} to the least square case, one obtains the following estimation
error bound from Eq.~\eqref{eqn_gQstar}
\begin{equation*}
\begin{aligned}
\|\beta^{(k)}-\beta^*\|^2 \leq \Omega(\bar{k}^2\|\triangledown
Q(\beta^*)\|_\infty^2) \leq \Omega({n^{-1} \bar{k}^2 \log d})
\end{aligned}
\end{equation*}
and the following condition of feature selection consistency: if $|{\beta}_j^*|\geq
\Omega({\sqrt{\bar{k}n^{-1}\log d}})~\forall~j\in \bar{F}$, then
\begin{equation*}
\begin{aligned}
\rm{supp}(\beta^{(k)})=\rm{supp}(\beta^*).
\end{aligned}
\end{equation*}
One can observe that the general analysis gives a looser bound for
estimation error and requires a stronger condition for feature
selection consistency than the analysis for the special case.

Our results in Theorems~\ref{thm_main1_obj} and~\ref{thm_main1} bridge this gap when combined with Eqs.~\eqref{eq:secResult1} and \eqref{eq:secResult2}.
The first inequalities in Theorems~\ref{thm_main1_obj} and~\ref{thm_main1} indicate that
{\small
\begin{align*}
&\|\beta^{(k)} - \beta^*\|^2 \leq (\|\beta^{(k)}-\bar\beta\| + \|\bar\beta - \beta^*\|)^2 \\
\leq &\Omega\bigg(n^{-1}(\bar{k}+ \log d~|~\{j\in \bar{F}-F^{(k)}:\\
&~|\bar{\beta}_j|<\Omega(n^{-1/2}\sqrt{\log d})\}|)\bigg)~~~~[\text{from Eq.~\eqref{eq:secResult1}}]\\
\leq &\Omega\bigg(n^{-1}(\bar{k}+ \log d~|~\{j\in \bar{F}-F^{(k)}:\\
&~|\beta^*_j|<\Omega(n^{-1/2}\sqrt{\log d})\}|)\bigg)~~~~[\text{from
Eq.~\eqref{eq:secResult2}}]
\end{align*}
}
which is consistent with the results in Eq.~\eqref{eq:secIntro1}. The
last inequality in Theorem~\ref{thm_main1} also implies that feature selection consistency is guaranteed as well, as long as
$|\bar\beta_j|>\Omega(\sqrt{n^{-1}\log d})$ (or $|\beta^*_j|>\Omega(\sqrt{n^{-1}\log d})$) for all $j\in \bar{F}$. This
requirement agrees with the results in
Eq.~\eqref{eq:secIntro2}. 

\section{Application: Sensor Selection for Human Activity Recognition}
Machine learning technologies for smart home systems and home energy
management systems have recently attracted much attention. Among the
many promising applications such as optimal energy control,
emergency alerts for elderly persons living alone, and automatic
life-logging, a fundamental challenge for these applications is to
recognize human activity at homes, with the smallest number of
sensors. The data mining task here is to minimize the number of
sensors without significantly worsening recognition accuracy. We
used pyroelectric sensors, which return binary signals in reaction
to human motion.

Fig.~\ref{fig:sensor position} shows our experimental room layout
and sensor locations. The numbers represent sensors, and the
ellipsoid around each represents the area covered by it. We used 40 sensors, i.e., we observe a 40-dimensional binary time series. A single person lives in the room
for roughly one month, and data is collected on the basis of manually tagging his
activities into the pre-determined 14 categories summarized in
Table~\ref{table:activities}. For data preparation
reasons, we use the first 20\%~(roughly one week) samples in the
data, and divide it into 10\% for training and 10\% for testing. The
numbers of training and test samples are given in
Table~\ref{fig:sensor position}.

Pyroelectric sensors are preferable over cameras for two
practical reasons: cameras tend to create a psychological barrier
and pyroelectric sensors are much cheaper and easier to implement at
homes. Such sensors only observe noisy binary location information.
This means that, for high recognition accuracy, history (sequence)
information must be taken into account. The binary time series data follows a linear-chain conditional random field (CRF)~\citep{LaffertyMP01, Sutton06}. Linear-chain CRF gives a smooth and convex loss function; see Appendix~\ref{sec:CRFsyn} for more details of CRF.

Our task then is sensor selection on the basis of noisy binary time
series data, and to do this we apply our FoBa-gdt-CRF (FoBa-gdt with CRF objective function). Since it is very expensive to evaluate the CRF objective value and its gradient, FoBa-obj becomes impractical in this case (a large number of optimization problems in the forward step make it computationally very expensive). Here, we
consider a sensor to have been ``used'' if at least one feature
related to it is used in the CRF. Note that we have 14
activity-signal binary features~(i.e., indicators of sensor/activity
simultaneous activations) for each single sensor, and therefore we
have $40 \times 14 = 560$ such features in total. In addition, we
have {\lt $14 \times 14 = 196$} activity-activity binary
features~(i.e., indicators of the activities at times $t-1$ and $t$).
As explained in Section~\ref{sec:CRFsyn}, we only enforced sparsity
on the first type of features.

\begin{figure}[t]
  \centering
 \vspace{-3mm}
   \subfigure{\includegraphics[width=0.45\textwidth, height=0.45\textwidth]{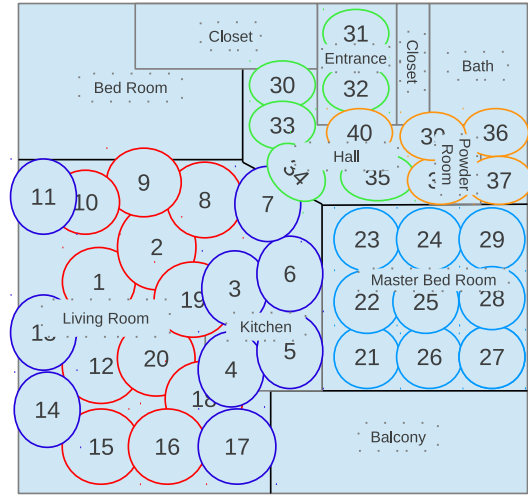}}    
\vspace{-10mm}
\caption{Room layout and sensor locations.}    \label{fig:sensor position}
\vspace{-5mm}
\end{figure}

\begin{table}[t]
{\tiny
\caption{Activities in the sensor data set}
\begin{center}
  \begin{tabular}{  c  c  c  }
   \\ \toprule
   ID & Activity & train / test samples\\ \midrule
   1 & Sleeping & 81K / 87K \\
   2  &  Out of Home (OH) & 66K / 42K  \\
   3  &  Using Computer & 64K / 46K  \\
   4  &  Relaxing & 25K / 65K  \\
   5  &  Eating &  6.4K / 6.0K  \\
   6  &  Cooking  & 5.2K / 4.6K  \\
   7  &  Showering (Bathing) & 3.9K / 45.0K \\
   8  &  No Event& 3.4K / 3.5K \\
   9  &  Using Toilet & 2.5K / 2.6K  \\
   10  &  Hygiene (brushing teeth, etc.)& 1.6K / 1.6K \\
   11  &  Dishwashing &  1.5K /1.8K   \\
   12 &  Beverage Preparation & 1.4K / 1.4K \\
   13 &  Bath Cleaning/Preparation & 0.5K / 0.3K  \\
   14 &  Others  & 6.5K / 2.1K \\ \midrule
   Total &  -  & 270K / 270K \\ \bottomrule
  \end{tabular}
  \end{center}
}  
\label{table:activities}
\vspace{-5mm}
\end{table}

First we compare FoBa-gdt-CRF with Forward-gdt-CRF (Forward-gdt
with CRF loss function) and L1-CRF\footnote{L1-CRF solves the
optimization problem with CRF loss + L1 regularization. Since it is
difficult to search the whole space L1 regularization parameter value space, we
investigated a number of discrete values.} in terms of test
recognition error over the number of sensors selected~(see the top
of Fig.~\ref{fig:sensor accuracy}). We can observe that
\begin{itemize}[noitemsep,nolistsep,leftmargin=*]
\item The performance for all methods get improved when the umber of sensors increases.
\item FoBa-gdt-CRF and Forward-gdt-CRF achieve comparable performance. However, FoBa-gdt-CRF reduces the error rate slightly faster, in terms of the number of sensors.
\item FoBa-gdt-CRF achieves its best performance with 14-15 sensors while Forward-gdt-CRF needs 17-18 sensors to achieve the same error level. We obtain sufficient accuracy by using fewer than 40 sensors.
\item FoBa-gdt-CRF consistently requires fewer features than Forward-gdt-CRF to achieve the same error level when using the same number of sensors.
\end{itemize}
\begin{figure}[t]
  \centering
    \subfigure{\includegraphics[scale=0.32]{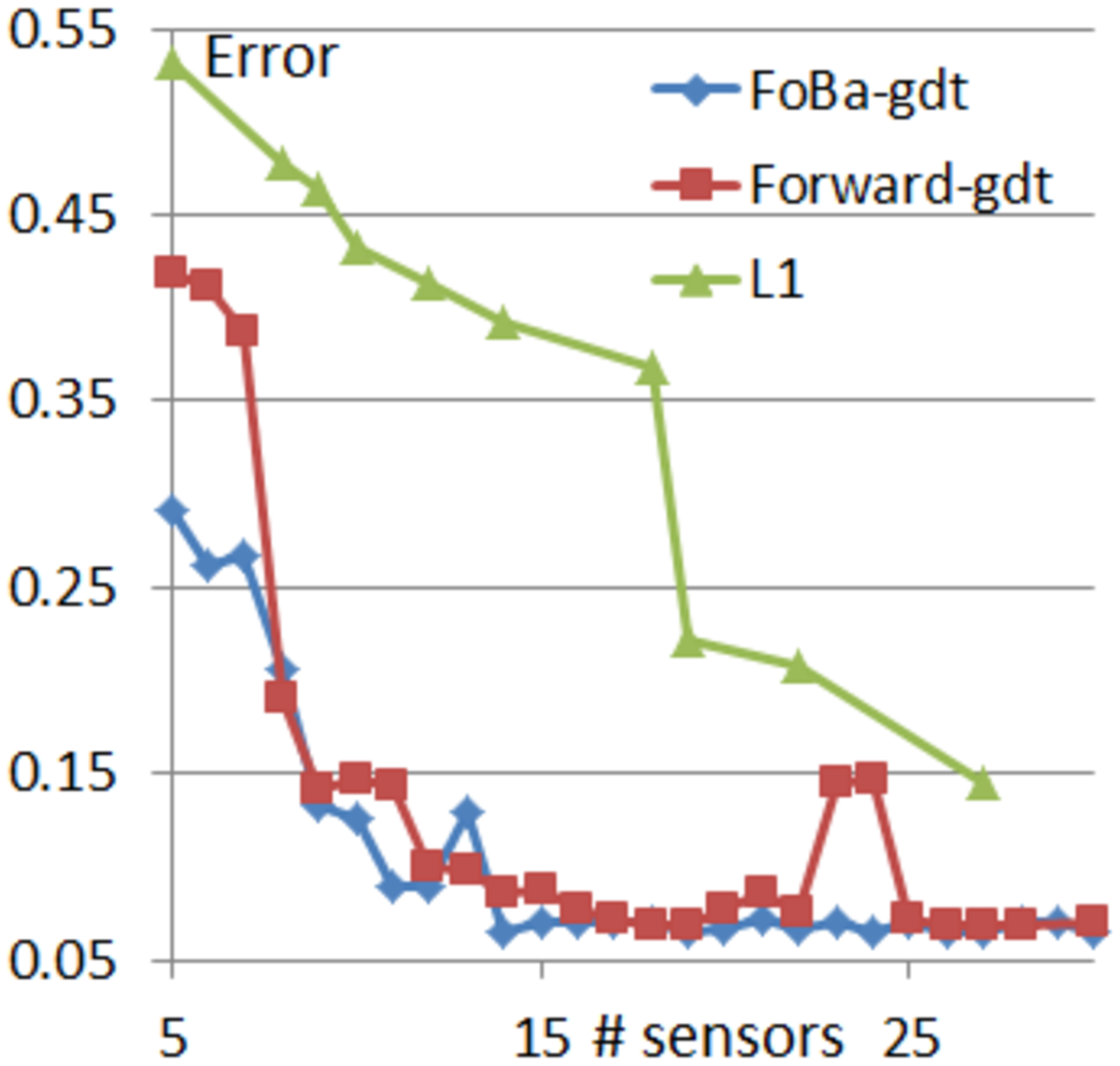}} 
    \subfigure{\includegraphics[scale=0.32]{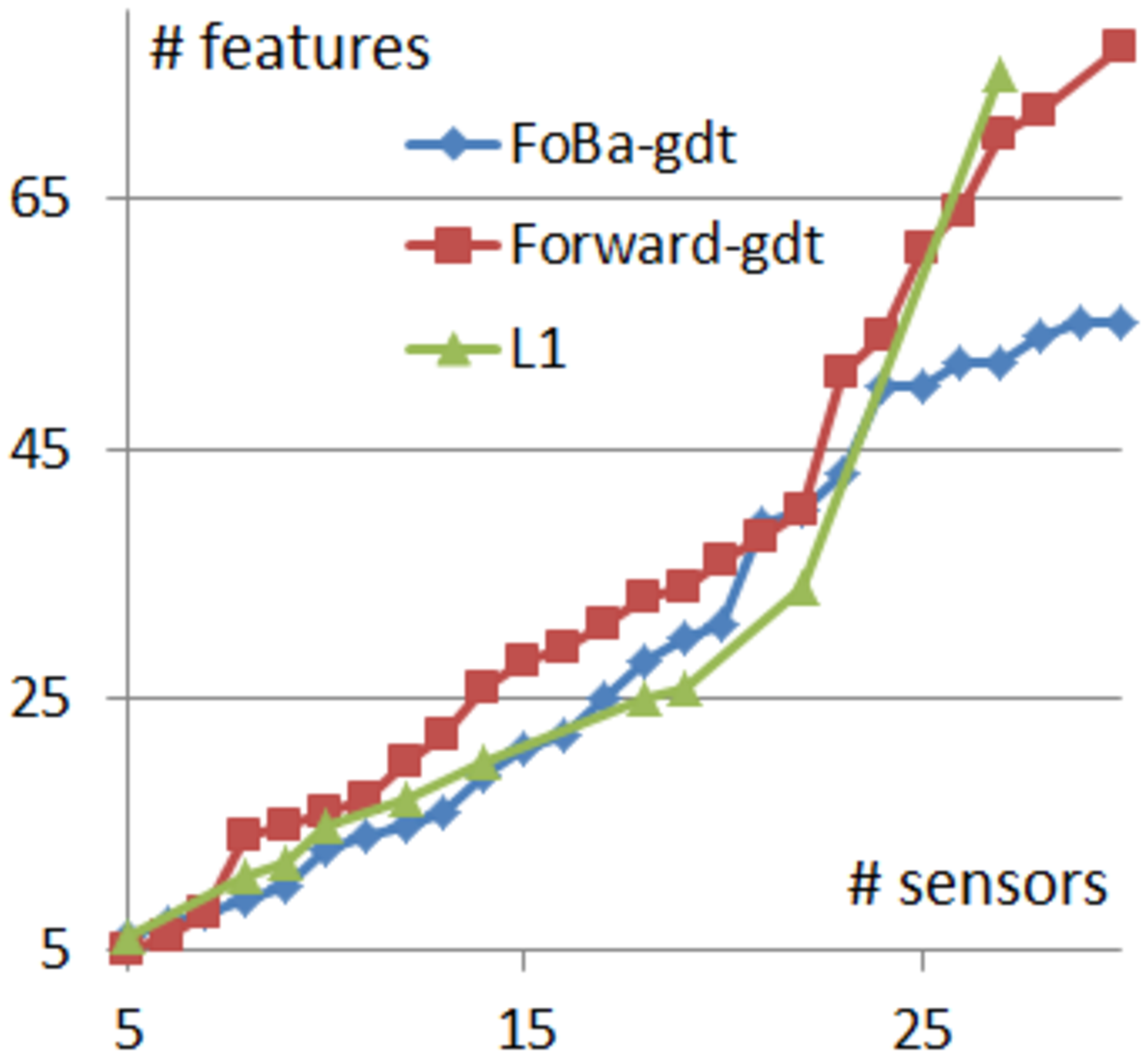}}\\
    \subfigure{\includegraphics[scale=0.80]{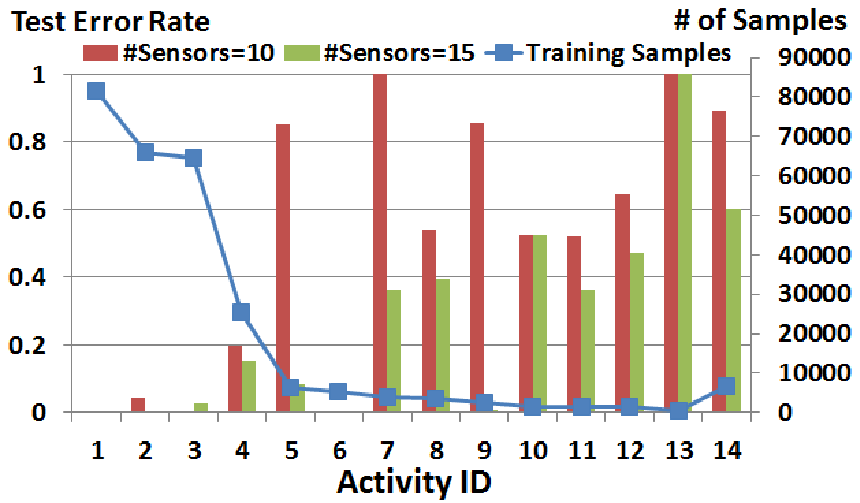}}
    \vspace{-5mm}
    \caption{Top: comparisons of FoBa-gdt-CRF, Forward-gdt-CRF and L1-CRF. Bottom: test error rates~(FoBa-gdt-CRF) for individual activities.}
    \label{fig:sensor accuracy}
\vspace{-5mm}
\end{figure}

We also analyze the test error rates of FoBa-gdt-CRF for individual
activities. We consider two cases with the number of sensors being 10 and 15, and entered
their test error rates for each individual activity in the bottom of
Fig.~\ref{fig:sensor accuracy}. We observe that:
\begin{itemize}[noitemsep,nolistsep,leftmargin=*]
\item The high frequency activities (e.g., activities \{1,2,3,4\}) are well recognized in both cases. In other words, FoBa-gdt-CRF is likely to select sensors (features) which contribute to the discrimination of high frequency activities.
\item The error rates for activities \{5, 7, 9\} significantly improve when the number of sensors increases from 10 to 15. Activities 7 and 9 are {\it Showering} and {\it Using Toilet}, and the use of additional sensors \{36, 37, 40\} {\lt seems} to have contributed to this improvement. 
Also, a dinner table was located near sensor 2, which is why the error rate w.r.t. activity 5 (Eating) significantly decreases from the case \# of sensors=10 to \# of sensors=15 by including sensor 2. 
\end{itemize}

\begin{table}[t]
{\tiny
\begin{center}
\caption{Sensor IDs selected by FoBa-gdt-CRF.}
\centering
  \begin{tabular}{  c  c  }
   \\ \toprule
   \# of sensors=10 &  \{1, 4, 5, 9, 10, 13, 19, 28, 34, 38\} \\
   \# of sensors=15 &  \{\# of sensors=10\} + \{2, 7, 36, 37, 40\} \\ \bottomrule
  \end{tabular}
\end{center}
}
\label{table:selected sensors}
\vspace{-5mm}
\end{table}


\section{Conclusion}

This paper considers two forward-backward greedy methods, a state-of-the-art greedy method FoBa-obj and its variant FoBa-gdt which is more efficient than FoBa-obj, for solving sparse feature selection problems with general convex smooth functions. We systematically analyze the theoretical properties of both algorithms. Our main contributions include: (i) We derive better theoretical bounds for FoBa-obj and FoBa-gdt than existing analyses regarding FoBa-obj in \citep{Jalali11} for general smooth convex functions. Our result also suggests that the NP hard problem~\eqref{eq:intro1} can be solved by FoBa-obj and FoBa-gdt if the signal noise ratio is big enough; (ii) Our new bounds are consistent with the bounds of a special case (least squares)~\citep{Zhang11} and fills a previously existing theoretical gap for general convex smooth functions \citep{Jalali11}; (iii) We provide the condition to satisfy the restricted strong convexity condition in commonly used machine learning problems; 
(iv) We apply FoBa-gdt (with the conditional random field objective) to the sensor selection problem for human indoor activity recognition and our results show that FoBa-gdt can successfully remove unnecessary sensors and is able to select more valuable sensors than other methods (including the ones based on forward greedy selection and L1-regularization). As for the future work, we plan to extend FoBa algorithms to minimize a general convex smooth function over a low rank constraint.  



\section{Acknowledgements}
We would like to sincerely thank Professor Masamichi Shimosaka of the
University of Tokyo for providing sensor data collected in
his research and Professor Stephen Wright of the University of
Wisconsin-Madison for constructive comments and helpful advice.
The majority of the work reported here was done during the internship of the first
author at NEC Laboratories America, Cupertino, CA.

\clearpage
{
\bibliographystyle{icml2014}
\bibliography{reference}
}

\clearpage
\appendix
\onecolumn{
\begin{center}
{\huge\bf
Appendix: Proofs
}
\end{center}
This appendix provides the proofs for our main results including the properties of FoBa-obj (Theorems~\ref{thm_main2_obj}, and \ref{thm_main1_obj}) and FoBa-gdt (Theorems~\ref{thm_main2} and \ref{thm_main1}), and analysis for the RSCC condition (Theorem~\ref{thm_kappa}).

Our proofs for Theorems~\ref{thm_main2_obj}, \ref{thm_main2}, \ref{thm_main1_obj}, and \ref{thm_main1} borrowed many tools from early literatures including \citet{Zhang11a}, \citet{Jalali11}, \citet{JohnsonJR12}, \citet{Zhang09}, \citet{Zhang11}. The key novelty in our proof is to develop a key property for the backward step (see Lemmas~\ref{lem_gbc_obj} and~\ref{lem_gbc}), which gives an upper bound for the number of wrong features included in the feature pool. By taking a full advantage of this upper bound, the presented proofs for our main results (i.e., Theorems~\ref{thm_main2_obj}, \ref{thm_main2}, \ref{thm_main1_obj}, and \ref{thm_main1}) are significantly simplified. It avoids the complicated induction procedure in the proof for the forward greedy method \citep{Zhang11a} and also improves the existing analysis for the same problem in \citet{Jalali11}. 

\section{Proofs of FoBa-obj}
First we prove the properties of the FoBa-obj algorithm, particularly Theorems~\ref{thm_main1_obj} and \ref{thm_main2_obj}.

Lemma~\ref{lem_obj_1} and Lemma~\ref{lem_obj_2} build up the dependence between the objective function change $\delta$ and the gradient $|\nabla Q(\beta)_j|$. Lemma~\ref{lem_gbc_obj} studies the effect of the backward step. Basically, it gives an upper bound of $|F^{(k)}-\bar{F}|$, which meets our intuition that the backward step is used to optimize the
size of the feature pool. Lemma~\ref{lem_bark_obj} shows that if $\delta$ is big enough, which
means that the algorithm terminates early, then $Q(\beta^{(k)})$ cannot be smaller than $Q(\bar\beta)$. Lemma~\ref{lem_gfs_obj} studies the forward step of FoBa-obj. It shows that the objective decreases sufficiently in every forward step, which together with Lemma~\ref{lem_bark_obj} (that is, $Q(\beta^{(k)})>Q(\bar\beta)$), implies that the algorithm will terminate within limited number of iterations. Based on these results, we provide the complete proofs of Theorems~\ref{thm_main1_obj} and~\ref{thm_main2_obj} at the end of this section. 

\begin{lemma} \label{lem_obj_1}
If $Q(\beta)-\min_\eta Q(\beta+\eta e_j) \geq \delta$, then we have
$$|\nabla Q(\beta)_j|\geq \sqrt{2\rho_-(1)\delta}.$$
\end{lemma}
\begin{proof}
From the condition, we have
\begin{align*}
-\delta \geq & \min_{\eta} Q(\beta+\eta e_j)-Q(\beta) \\
\geq & \min_{\eta} \langle \eta e_j, \nabla Q(\beta) \rangle + {\rho_-(1)\over 2}\eta^2 ~~~~(\text{from the definition of $\rho_-(1)$})\\
= & \min_{\eta} {\rho_-(1)\over 2}\left(\eta-{\nabla Q(\beta)_j\over \rho_-(1)}\right)^2 - {|\nabla Q(\beta)_j|^2\over 2\rho_-(1)} \\
= &-{|\nabla Q(\beta)_j|^2\over 2\rho_-(1)}.
\end{align*}
It indicates that $|\nabla Q(\beta)_j| \geq
\sqrt{2\rho_-(1)\delta}$.
\end{proof}
%

\begin{lemma} \label{lem_obj_2}
If $Q(\beta)-\min_{j,\eta} Q(\beta+\eta e_j) \leq \delta$, then we
have $$\|\nabla Q(\beta)\|_\infty\leq \sqrt{2\rho_+(1)\delta}.$$
\end{lemma}
\begin{proof}
From the condition, we have
\begin{align*}
\delta \geq & Q(\beta)-\min_{\eta,j} Q(\beta+\eta e_j) \\
=&\max_{\eta,j} Q(\beta) - Q(\beta+\eta e_j)\\
\geq & \max_{\eta,j} -\langle \eta e_j, \nabla Q(\beta) \rangle - {\rho_+(1)\over 2}\eta^2 \\
= & \max_{\eta,j} -{\rho_+(1)\over 2}\left(\eta-{\nabla Q(\beta)_j\over \rho_+(1)}\right)^2 - {|\nabla Q(\beta)_j|^2\over 2\rho_+(1)} \\
= &\max_j{|\nabla Q(\beta)_j|^2\over 2\rho_+(1)}\\
= &{\|\nabla Q(\beta)\|^2_\infty \over 2\rho_+(1)}
\end{align*}
It indicates that $\|\nabla Q(\beta)\|_\infty \leq
\sqrt{2\rho_+(1)\delta}$.
\end{proof}

\begin{lemma} \label{lem_gbc_obj}
(General backward criteria). Consider $\beta^{(k)}$ with the support
$F^{(k)}$ in the beginning of each iteration in Algorithm~\ref{alg_FoBaobj} with FoBa-obj (Here, $\beta^{(k)}$ is not necessarily the output of this algorithm). We have for any $\bar\beta\in
\mathbb{R}^{d}$ with the support $\bar{F}$
\begin{equation}
\|\beta^{(k)}_{F^{(k)}-\bar{F}}\|^2 =
\|(\beta^{(k)}-\bar\beta)_{F^{(k)}-\bar{F}}\|^2 \geq
{\delta^{(k)}\over \rho_+(1)}|F^{(k)}-\bar{F}|\geq {\delta\over
\rho_+(1)}|F^{(k)}-\bar{F}|.
\end{equation}
\end{lemma}
\begin{proof}
We have
\begin{align*}
|F^{(k)}-\bar{F}|\min_{j\in F^{(k)}} Q(\beta^{(k)}-\beta^{(k)}_j e_j) \leq & \sum_{j\in F^{(k)}-\bar{F}} Q(\beta^{(k)}-\beta^{(k)}_je_j)\\
\leq & \sum_{j\in F^{(k)}-\bar{F}} Q(\beta^{(k)}) - \triangledown Q(\beta^{(k)})_j\beta^{(k)}_j + {\rho_+(1)\over 2}(\beta^{(k)}_j)^2 \\
\leq & |F^{(k)}-\bar{F}|Q(\beta^{(k)}) + {\rho_+(1)\over 2}
\|(\beta^{(k)}-\bar\beta)_{F^{(k)}-\bar{F}}\|^2.
\end{align*}
The second inequality is due to $\triangledown Q(\beta^{(k)})_j=0$
for any $j\in F^{(k)}-\bar{F}$ and the third inequality is from
$\bar{\beta}_{F^{(k)}-\bar{F}}=0$. It follows that $${\rho_+(1)\over
2} \|(\beta^{(k)}-\bar\beta)_{F^{(k)}-\bar{F}}\|^2\geq
|F^{(k)}-\bar{F}|(\min_{j\in F^{(k)}} Q(\beta^{(k)}-\beta^{(k)}_j
e_j) - Q(\beta^{(k)}))\geq |F^{(k)}-\bar{F}|{\delta^{(k)}\over 2},$$
which implies that the claim using $\delta^{(k)}\geq \delta$.
\end{proof}

\begin{lemma}\label{lem_bark_obj}
  Let $\bar\beta = \arg\min_{supp(\beta)\in \bar{F}}Q(\beta)$. Consider $\beta^{(k)}$ in the beginning of each iteration in Algorithm with FoBa-obj. Denote its support as $F^{(k)}$. Let $s$ be any integer larger than $|F^{(k)}-\bar{F}|$. If takes $\delta > {4\rho_+(1)\over \rho_-(s)^2}\|\triangledown Q(\bar\beta)\|_\infty^2$ in FoBa-obj, then we have $Q(\beta^{(k)})\geq Q(\bar\beta)$.
\end{lemma}
\begin{proof}
We have
  \begin{align*}
    & Q(\beta^{(k)}) - Q(\bar{\beta}) \\
    \geq & \langle \triangledown Q(\bar\beta), \beta^{(k)}-\bar\beta \rangle + {\rho_-(s)\over 2}\|\beta^{(k)}-\bar\beta\|^2\\
    \geq & \langle \triangledown Q(\bar\beta)_{F^{(k)}-\bar{F}}, (\beta^{(k)}-\bar\beta)_{F^{(k)}-\bar{F}} \rangle + {\rho_-(s)\over 2}\|\beta^{(k)}-\bar\beta\|^2\\
    \geq & -\|\triangledown Q(\bar\beta)\|_\infty \|(\beta^{(k)}-\bar\beta)_{F^{(k)}-\bar{F}}\|_1 + {\rho_-(s)\over 2}\|\beta^{(k)}-\bar\beta\|^2\\
    \geq & -\|\triangledown Q(\bar\beta)\|_\infty\sqrt{|F^{(k)}-\bar{F}|} \|\beta^{(k)}-\bar\beta\| + {\rho_-(s)\over 2}\|\beta^{(k)}-\bar\beta\|^2\\
    \geq & -\|\triangledown Q(\bar\beta)\|_\infty \|\beta^{(k)}-\bar\beta\|^2{\sqrt{\rho_+(1) \over \delta}} + {\rho_-(s)\over 2}\|\beta^{(k)}-\bar\beta\|^2~~~~(\text{from Lemma~\ref{lem_gbc_obj}})\\
    =& \left({\rho_-(s)\over 2}-{\|\triangledown Q(\bar\beta)\|_\infty \sqrt{\rho_+(1)}\over \sqrt{\delta}}\right)\|\beta^{(k)}-\bar\beta\|^2\\
    >& 0.~~~~(\text{from $\delta > {4\rho_+(1)\over \rho_-(s)^2}\|\triangledown Q(\bar\beta)\|_\infty^2$})
  \end{align*}
 It proves the claim.
\end{proof}

\noindent{\bf Proof of Theorem~\ref{thm_main1_obj}}
\begin{proof}
  From Lemma~\ref{lem_bark_obj}, we only need to consider the case $Q(\beta^{(k)})\geq Q(\bar\beta)$. We have
  \begin{align*}
    0\geq Q(\bar\beta)-Q(\beta^{(k)}) \geq \langle \triangledown Q(\beta^{(k)}), \bar\beta-\beta^{(k)} \rangle + {\rho_{-}(s) \over 2}\|\bar{\beta}-\beta^{(k)}\|^2,
  \end{align*}
  which implies that
  \begin{align*}
    {\rho_-(s)\over 2} \|\bar\beta-\beta^{(k)}\|^2 \leq & - \langle \triangledown Q(\beta^{(k)}), \bar\beta-\beta^{(k)} \rangle\\
    =&-\langle \triangledown Q(\beta^{(k)})_{\bar{F}-F^{(k)}}, (\bar\beta-\beta^{(k)})_{\bar{F}-F^{(k)}} \rangle\\
    \leq & \|\triangledown Q(\beta^{(k)})_{\bar{F}-F^{(k)}}\|_\infty \|(\bar\beta-\beta^{(k)})_{\bar{F}-F^{(k)}}\|_1 \\
    \leq & \sqrt{2\rho_+(1)\delta}
    \sqrt{|\bar{F}-F^{(k)}|}|\bar\beta-\beta^{(k)}\|.~~~~(\text{from Lemma~\ref{lem_obj_2}})
  \end{align*}
  We obtain
  \begin{equation}
  \|\bar\beta-\beta^{(k)}\|\leq {2\sqrt{2\rho_+(1)\delta}\over \rho_-(s)}\sqrt{|\bar{F}-F^{(k)}|}.
  \label{eqn_thmmain_1_obj}
  \end{equation}
  It follows that
  \begin{align*}
    {8\rho_+(1)\delta\over \rho_-(s)^2}|\bar{F}-F^{(k)}| \geq & \|\bar{\beta}-\beta^{(k)}\|^2 \\
    \geq & \|(\bar\beta-\beta^{(k)})_{\bar{F}-F^{(k)}}\|^2 \\
    = & \|\bar{\beta}_{\bar{F}-F^{(k)}}\|^2 \\
    \geq & \gamma^2 |\{j\in \bar{F}-F^{(k)}:~|\bar{\beta}_j|\geq\gamma\}|\\
    =& {16\rho_+(1)\delta\over \rho_-(s)^2}|\{j\in \bar{F}-F^{(k)}:~|\bar{\beta}_j|\geq\gamma\}|,
  \end{align*}
  which implies that
  \begin{align*}
    &|\bar{F}-F^{(k)}|\geq 2|\{j\in \bar{F}-F^{(k)}:~|\bar{\beta}_j|\geq \gamma\}| = 2(|\bar{F}-F^{(k)}|-|\{j\in \bar{F}-F^{(k)}:~|\bar{\beta}_j|<\gamma\}|)\\
    \Rightarrow & |\bar{F}-F^{(k)}|\leq 2|\{j\in \bar{F}-F^{(k)}:~|\bar{\beta}_j|< \gamma\}|.
  \end{align*}
  The first inequality is obtained from Eq.~\eqref{eqn_thmmain_1_obj} $$\|\bar\beta-\beta^{(k)}\|\leq {2\sqrt{2\rho_+(1)\delta}\over \rho_-(s)}\sqrt{|\bar{F}-F^{(k)}|}\leq {4\sqrt{\rho_+(1)\delta}\over \rho_-(s)}\sqrt{|j\in \bar{F}-F^{(k)}: |\bar{\beta}_j|<\gamma|}.$$

  Next, we consider
  \begin{align*}
    &Q(\bar\beta)-Q(\beta^{(k)})\\
    \geq & \langle \triangledown Q(\beta^{(k)}), \bar\beta-\beta^{(k)} \rangle + {\rho_{-}(s) \over 2}\|\bar{\beta}-\beta^{(k)}\|^2\\
    = & \langle \triangledown Q(\beta^{(k)})_{\bar{F}-F^{(k)}}, (\bar\beta-\beta^{(k)})_{\bar{F}-F^{(k)}} \rangle + {\rho_{-}(s) \over 2}\|\bar{\beta}-\beta^{(k)}\|^2~~~~(\text{from~$\triangledown Q(\beta^{(k)})_{F^{(k)}}=0$})\\
    \geq & -\|\triangledown Q(\beta^{(k)})\|_\infty\|\bar{\beta}-\beta^{(k)}\|_1 + {\rho_{-}(s) \over 2}\|\bar{\beta}-\beta^{(k)}\|^2 \\
    \geq & -\sqrt{2\rho_+(1)\delta}\sqrt{|\bar{F}-F^{(k)}|}\|\bar{\beta}-\beta^{(k)}\| + {\rho_{-}(s) \over 2}\|\bar{\beta}-\beta^{(k)}\|^2\\
    \geq &{\rho_-(s)\over 2}\left(\|\bar\beta-\beta^{(k)}\| - \sqrt{2\rho_+(1)\delta}\sqrt{|\bar{F}-F^{(k)}|}/\rho_-(s)\right)^2 - 2\rho_+(1)\delta|\bar{F}-F^{(k)}|/(2\rho_-(s))\\
    \geq & - 2\rho_+(1)\delta|\bar{F}-F^{(k)}|/(2\rho_-(s))\\
    \geq & - {2\rho_+(1)\delta\over \rho_-(s)}|\{j\in \bar{F}-F^{(k)}:~|\bar{\beta}_j|<\gamma\}|.
  \end{align*}
  It proves the second inequality.
The last inequality in this theorem is obtained from the following:
  \begin{align*}
    &{2\rho_+(1)\delta\over 2\rho_+(1)^2}|F^{(k)}-\bar{F}|\\
    \leq &\|(\beta^{(k)}-\bar\beta)_{F^{(k)}-\bar{F}}\|^2 ~~~~(\text{from Lemma~\ref{lem_gbc_obj}})\\
    \leq &\|\beta^{(k)}-\bar\beta\|^2 \\
    \leq &{8\rho_+(1)\delta\over \rho_-(s)^2}|\bar{F}-F^{(k)}|\\
    \leq &{16\rho_+(1)\delta\over \rho_-(s)^2}|\{j\in \bar{F}-F^{(k)}: |\bar{\beta}_j|<\gamma\}|.
  \end{align*}
  It completes the proof.
\end{proof}

\begin{lemma}\label{lem_gfs_obj}
(General forward step of FoBa-obj) Let $\beta =
\arg\min_{supp(\beta)\subset F} Q(\beta)$. For any $\beta'$ with
support $F'$ and $i\in\{i:~Q(\beta)-\min_\eta Q(\beta + \eta
e_i)\geq(Q(\beta)-\min_{j,\eta} Q(\beta + \eta e_j))\}$, we have
$$|F'-F|(Q(\beta)-\min_{\eta}Q(\beta+\eta e_i))\geq {\rho_-(s)\over \rho_+(1)}(Q(\beta)-Q(\beta')).$$
\end{lemma}
\begin{proof}
We have that the following holds with any $\eta$
\begin{align*}
&|F'-F|\min_{j\in F'-F, \eta} Q(\beta+\eta (\beta_j'-\beta_j) e_j)\\
\leq & \sum_{j\in F'-F} Q(\beta+\eta (\beta'_j-\beta_j) e_j)\\
\leq & \sum_{j\in F'-F} Q(\beta) + \eta(\beta'_j-\beta_j)\nabla
Q(\beta)_j
 + {\rho_+(1)\over 2}(\beta_j-\beta'_j)^2\eta^2 \\
\leq&|F'-F|Q(\beta) + \eta\langle (\beta'-\beta)_{F'-F}, \nabla Q(\beta)_{F'-F} \rangle + {\rho_+(1)\over 2}\|\beta'-\beta\|^2\eta^2\\
=&|F'-F|Q(\beta) + \eta\langle \beta'-\beta, \nabla Q(\beta) \rangle + {\rho_+(1)\over 2}\|\beta'-\beta\|^2\eta^2\\
\le&|F'-F|Q(\beta) + \eta(Q(\beta')-Q(\beta)-{\rho_-(s)\over
2}\|\beta-\beta'\|^2) + {\rho_+(1)\over 2}\|\beta'-\beta\|^2\eta^2.
\end{align*}
By minimizing $\eta$, we obtain that
\begin{equation}
\begin{aligned}
&|F'-F|(\min_{j\in F'-F, \eta} Q(\beta+\eta (\beta_j'-\beta_j) e_j)-Q(\beta)) \\
\leq & \min_\eta \eta(Q(\beta')-Q(\beta)-{\rho_-(s)\over 2}\|\beta-\beta'\|^2) + {\rho_+(1)\over 2}\|\beta'-\beta\|^2\eta^2\\
\leq & -{(Q(\beta')-Q(\beta)-{\rho_-(s)\over 2}\|\beta-\beta'\|^2)^2\over 2\rho_+(1)\|\beta'-\beta\|^2} \\
\leq & -{4(Q(\beta)-Q(\beta')){\rho_-(s)\over 2}\|\beta-\beta'\|^2\over 2\rho_+(1)\|\beta'-\beta\|^2}~~~~(\text{from}~(a+b)^2\geq 4ab)\\
= & {\rho_-(s)\over \rho_+(1)}(Q(\beta')-Q(\beta)).
\label{eqn_lemgfs_obj_1}
\end{aligned}
\end{equation}
It follows that
\begin{align*}
&|F'-F|(Q(\beta)-\min_{\eta, j} Q(\beta + \eta e_i))\\
\geq &|F'-F|(Q(\beta)-\min_{j\in F'-F,\eta} Q(\beta + \eta e_j)) \\
\geq &|F'-F|(Q(\beta)-\min_{j\in F'-F,\eta} Q(\beta + \eta (\beta'_j-\beta_j) e_j)) \\
\geq &{\rho_-(s)\over \rho_+(1)}(Q(\beta)-Q(\beta')),~~~~(\text{from Eq.~\eqref{eqn_lemgfs_obj_1}}) \\
\end{align*}
which proves the claim.
\end{proof}

\noindent {\bf Proof of Theorem~\ref{thm_main2_obj}}
\begin{proof}
Assume that the algorithm terminates at some number larger than
$s-\bar{k}$. Then we consider the first time $k=s-\bar{k}$. Denote
the support of $\beta^{(k)}$ as $F^{(k)}$. Let $F' = \bar{F}\cup
F^{(k)}$ and $\beta' = \arg\min_{supp(\beta)\subset F'}Q(\beta)$.
One can easily verify that $|\bar{F}\cup F^{(k)}|\leq s$. We
consider the very recent $(k-1)$ step and have
\begin{equation}
  \begin{aligned}
    \delta^{(k)}= & Q(\beta^{(k-1)})-\min_{\eta,i}Q(\beta^{(k-1)}+\eta e_i)\\
    \geq & {\rho_-(s)\over \rho_+(1)|F'-F^{(k-1)}|}(Q(\beta^{(k-1)})-Q(\beta'))\\
    \geq & {\rho_-(s)\over \rho_+(1)|F'-F^{(k-1)}|}(Q(\beta^{(k)})-Q(\beta'))\\
    \geq & {\rho_-(s)\over \rho_+(1)|F'-F^{(k-1)}|}\left(\langle \triangledown Q(\beta'),~\beta^{(k)}-\beta'\rangle + {\rho_-(s)\over 2}\|\beta^{(k)}-\beta'\|^2\right)\\
    \geq & {\rho_-(s)^2\over 2\rho_+(1)|F'-F^{(k-1)}|}\|\beta^{(k)}-\beta'\|^2~~~~(\text{from $\triangledown Q(\beta')_{F'}=0$})
        \label{eq:deltaf}
    \end{aligned}
  \end{equation}
  where the first inequality is due to Lemma~\ref{lem_gfs_obj} and the second inequality is due to $Q(\beta^{(k)})\leq Q(\beta^{(k-1)})$. From Lemma~\ref{lem_gbc_obj}, we have
  \begin{equation}
    \delta^{(k)}\leq {2\rho_+(1)\over |F^{(k)}-\bar{F}|}\|\beta^{(k)}-\bar{\beta}\|^2={2\rho_+(1)\over |F'-\bar{F}|}\|\beta^{(k)}-\bar{\beta}\|^2
    \label{eq:deltab}
  \end{equation}
  Combining Eq.~\eqref{eq:deltaf} and \eqref{eq:deltab}, we obtain that
  \begin{align*}
    \|(\beta^{(k)}-\bar{\beta})\|^2 \geq \left(\rho_-(s)\over 2\rho_+(1)\right)^2{|F'-\bar{F}|\over |F'-F^{(k-1)}|}\|\beta^{(k)}-\beta'\|^2,
  \end{align*}
which implies that $$\|\beta^{(k)}-\bar{\beta}\|\geq
t\|\beta^{(k)}-\beta'\|,$$ where
\begin{align*}
t:= & {\rho_-(s)\over 2\rho_+(1)}\sqrt{|F'-\bar{F}|\over |F'-F^{(k-1)}|} \\
= & {\rho_-(s)\over 2\rho_+(1)}\sqrt{|F^{(k)} - F^{(k)}\cap \bar{F}|\over |F'-F^{(k)}| + 1}  \\
= & {\rho_-(s)\over 2\rho_+(1)}\sqrt{|F^{(k)} - F^{(k)}\cap \bar{F}| \over |\bar{F}- F^{(k)}\cap \bar{F}| + 1}  \\
= & {\rho_-(s)\over 2\rho_+(1)}\sqrt{|F^{(k)}| - |F^{(k)}\cap \bar{F}| \over |\bar{F}|- |F^{(k)}\cap \bar{F}| + 1}  \\
= & {\rho_-(s)\over 2\rho_+(1)}\sqrt{k - |F^{(k)}\cap \bar{F}| \over \bar{k}- |F^{(k)}\cap \bar{F}| + 1}  \\
\ge & {\rho_-(s)\over 2\rho_+(1)}\sqrt{(s-\bar{k})\over(\bar{k}+1)}\quad (\text{from $k=s-\bar{k}$ and $s\geq 2\bar{k}+1$})\\
\geq & \sqrt{\rho_+(s)\over
\rho_-(s)}+1~~~~(\text{from the assumption on $s$}).
\end{align*}
It follows
  \begin{align*}
    t\|\beta^{(k)}-\beta'\|\leq \|\beta^{(k)}-\bar\beta\| \leq \|\beta^{(k)}-\beta'\| + \|\beta'-\bar\beta\|~~~~(\text{from Eq.~\eqref{eq:thm_1}})
  \end{align*}
  which implies that
  \begin{equation}
  (t-1)\|\beta^{(k)}-\beta'\|\leq \|\beta'-\bar\beta\|.
  \label{eq:mainthm:3}
  \end{equation}
  Next we have
  \begin{align*}
    Q(\beta^{(k)})-Q(\bar{\beta}) =& Q(\beta^{(k)}) - Q(\beta') + Q(\beta') - Q(\bar\beta)\\
    \leq & {\rho_+(s)\over 2}\|\beta^{(k)}-\beta'\|^2 - {\rho_-(s)\over 2}\|\beta'-\bar\beta\|^2 \\
    \leq & (\rho_+(s)-\rho_-(s)(t-1)^2)/2\|\beta^{(k)}-\beta'\|^2\\
    \leq & 0.
  \end{align*}
Since the sequence $Q(\beta^{(k)})$ is strictly decreasing, we know
that $Q(\beta^{(k+1)})<Q(\beta^{(k)})\leq Q(\bar\beta)$.
  However, from Lemma~\ref{lem_bark_obj} we also have $Q(\bar{\beta})\leq Q(\beta^{(k+1)})$. Thus it leads to a contradiction. This indicates that the algorithm terminates at some integer not greater than $s-\bar{k}$.
\end{proof}

\section{Proofs of FoBa-gdt}
Next we consider Algorithm~\ref{alg_FoBaobj} with FoBa-gdt; specifically we will prove Theorem~\ref{thm_main1} and Theorem~\ref{thm_main2}. Our proof strategy is similar to FoBa-obj. 
\begin{lemma}
If $|\triangledown Q(\beta)_j| \geq \epsilon$, then we have
\begin{align*}
 Q(\beta) - \min_{\eta}Q(\beta+\eta e_j) \geq {\epsilon^2\over 2\rho_+(1)}.
\end{align*}
\end{lemma}
\begin{proof}
Consider the LHS:
  \begin{align*}
    &Q(\beta) - \min_\eta Q(\beta+\eta e_j) \\
    = & \max_{\eta} Q(\beta) - Q(\beta+\eta e_j)\\
    \geq & \max_{\eta} -\langle \triangledown Q(\beta), \eta e_j \rangle - {\rho_+(1)\over 2}\eta^2 \\
    = & \max_{\eta} -\eta \triangledown Q(\beta)_j - {\rho_+(1)\over 2}\eta^2 \\
    \geq & {|\triangledown Q(\beta)_j|^2 \over 2 \rho_+(1)}\\
    \geq & {\epsilon^2\over 2\rho_+(1)}.
  \end{align*}
It completes the proof.
\end{proof}
This lemma implies that $\delta^{(k_0)}\geq {\epsilon^2\over
2\rho_+(1)}$ for all $k_0=1,\cdots,k$, and $$Q(\beta^{(k_0-1)}) -
\min_{\eta}Q(\beta^{(k_0-1)}+\eta e_j) \geq \delta^{(k_0)} \geq
{\epsilon^2\over 2\rho_+(1)},$$if $|\triangledown
Q(\beta^{(k_0-1)})_j|\geq \epsilon$.

\begin{lemma} \label{lem_gbc}
(General backward criteria). Consider $\beta^{(k)}$ with the support
$F^{(k)}$ in the beginning of each iteration in Algorithm~\ref{alg_FoBaobj} with FoBa-obj. We have for any $\bar\beta\in
\mathbb{R}^{d}$ with the support $\bar{F}$
\begin{equation}
  \|\beta^{(k)}_{F^{(k)}-\bar{F}}\|^2 = \|(\beta^{(k)}-\bar\beta)_{F^{(k)}-\bar{F}}\|^2 \geq {\delta^{(k)}\over \rho_+(1)}|F^{(k)}-\bar{F}|\geq {\epsilon^2\over 2\rho_+(1)^2}|F^{(k)}-\bar{F}|.
\end{equation}
\end{lemma}
\begin{proof}
We have
\begin{align*}
|F^{(k)}-\bar{F}|\min_{j\in F^{(k)}} Q(\beta^{(k)}-\beta^{(k)}_j e_j) \leq & \sum_{j\in F^{(k)}-\bar{F}} Q(\beta^{(k)}-\beta^{(k)}_je_j)\\
\leq & \sum_{j\in F^{(k)}-\bar{F}} Q(\beta^{(k)}) - \triangledown Q(\beta^{(k)})_j\beta^{(k)}_j + {\rho_+(1)\over 2}(\beta^{(k)}_j)^2 \\
\leq & |F^{(k)}-\bar{F}|Q(\beta^{(k)}) + {\rho_+(1)\over 2}
\|(\beta^{(k)}-\bar\beta)_{F^{(k)}-\bar{F}}\|^2.
\end{align*}
The second inequality is due to $\triangledown Q(\beta^{(k)})_j=0$
for any $j\in F^{(k)}-\bar{F}$ and the third inequality is from
$\bar{\beta}_{F^{(k)}-\bar{F}}=0$. It follows that $${\rho_+(1)\over
2} \|(\beta^{(k)}-\bar\beta)_{F^{(k)}-\bar{F}}\|^2\geq
|F^{(k)}-\bar{F}|(\min_{j\in F^{(k)}} Q(\beta^{(k)}-\beta^{(k)}_j
e_j) - Q(\beta^{(k)}))\geq |F^{(k)}-\bar{F}|{\delta^{(k)}\over 2},$$
which implies the claim using $\delta^{(k)}\geq {\epsilon^2\over
2\rho_+(1)}$.
\end{proof}

\begin{lemma}\label{lem_bark}
  Let $\bar\beta = \arg\min_{supp(\beta)\in \bar{F}}Q(\beta)$. Consider $\beta^{(k)}$ in the beginning of each iteration in Algorithm~\ref{alg_FoBaobj} with FoBa-obj. Denote its support as $F^{(k)}$. Let $s$ be any integer larger than $|F^{(k)}-\bar{F}|$. If take $\epsilon > {2\sqrt{2}\rho_+(1)\over \rho_-(s)}\|\triangledown Q(\bar\beta)\|_\infty$ in FoBa-obj, then we have $Q(\beta^{(k)})\geq Q(\bar\beta)$.
\end{lemma}
\begin{proof}
 We have
  \begin{align*}
    0 > & Q(\beta^{(k)}) - Q(\bar{\beta}) \\
    \geq & \langle \triangledown Q(\bar\beta), \beta^{(k)}-\bar\beta \rangle + {\rho_-(s)\over 2}\|\beta^{(k)}-\bar\beta\|^2\\
    \geq & \langle \triangledown Q(\bar\beta)_{F^{(k)}-\bar{F}}, (\beta^{(k)}-\bar\beta)_{F^{(k)}-\bar{F}} \rangle + {\rho_-(s)\over 2}\|\beta^{(k)}-\bar\beta\|^2\\
    \geq & -\|\triangledown Q(\bar\beta)\|_\infty \|(\beta^{(k)}-\bar\beta)_{F^{(k)}-\bar{F}}\|_1 + {\rho_-(s)\over 2}\|\beta^{(k)}-\bar\beta\|^2\\
    \geq & -\|\triangledown Q(\bar\beta)\|_\infty\sqrt{|F^{(k)}-\bar{F}|} \|\beta^{(k)}-\bar\beta\| + {\rho_-(s)\over 2}\|\beta^{(k)}-\bar\beta\|^2\\
    \geq & -\|\triangledown Q(\bar\beta)\|_\infty \|\beta^{(k)}-\bar\beta\|^2{\sqrt{2} \rho_+(1)\over \epsilon} + {\rho_-(s)\over 2}\|\beta^{(k)}-\bar\beta\|^2~~~~(\text{from Lemma~\ref{lem_gbc}})\\
    =& \left({\rho_-(s)\over 2}-{\|\triangledown Q(\bar\beta)\|_\infty \sqrt{2} \rho_+(1)\over \epsilon}\right)\|\beta^{(k)}-\bar\beta\|^2\\
    >& 0.~~~~(\text{from $\epsilon > {2\sqrt{2}\rho_+(1)\over \rho_-(s)}\|\triangledown Q(\bar\beta)\|_\infty$})
  \end{align*}
  It proves our claim.
\end{proof}

\noindent {\bf Proof of Theorem~\ref{thm_main1}}
\begin{proof}
  From Lemma~\ref{lem_bark}, we only need to consider the case $Q(\beta^{(k)})\geq Q(\bar\beta)$. We have
  \begin{align*}
    0\geq Q(\bar\beta)-Q(\beta^{(k)}) \geq \langle \triangledown Q(\beta^{(k)}), \bar\beta-\beta^{(k)} \rangle + {\rho_{-}(s) \over 2}\|\bar{\beta}-\beta^{(k)}\|^2,
  \end{align*}
  which implies that
  \begin{align*}
    {\rho_-(s)\over 2} \|\bar\beta-\beta^{(k)}\|^2 \leq & - \langle \triangledown Q(\beta^{(k)}), \bar\beta-\beta^{(k)} \rangle\\
    =&-\langle \triangledown Q(\beta^{(k)})_{\bar{F}-F^{(k)}}, (\bar\beta-\beta^{(k)})_{\bar{F}-F^{(k)}} \rangle\\
    \leq & \|\triangledown Q(\beta^{(k)})_{\bar{F}-F^{(k)}}\|_\infty \|(\bar\beta-\beta^{(k)})_{\bar{F}-F^{(k)}}\|_1 \\
    \leq & \epsilon \sqrt{|\bar{F}-F^{(k)}|}|\bar\beta-\beta^{(k)}\|.
  \end{align*}
  We obtain
  \begin{equation}
  \|\bar\beta-\beta^{(k)}\|\leq {2\epsilon\over \rho_-(s)}\sqrt{|\bar{F}-F^{(k)}|}.
  \label{eqn_thmmain_1}
  \end{equation}
  It follows that
  \begin{align*}
    {4\epsilon^2\over \rho_-(s)^2}|\bar{F}-F^{(k)}| \geq & \|\bar{\beta}-\beta^{(k)}\|^2 \\
    \geq & \|(\bar\beta-\beta^{(k)})_{\bar{F}-F^{(k)}}\|^2 \\
    = & \|\bar{\beta}_{\bar{F}-F^{(k)}}\|^2 \\
    \geq & \gamma^2 |\{j\in \bar{F}-F^{(k)}:~|\bar{\beta}_j|\geq\gamma\}|\\
    =& {8\epsilon^2\over \rho_-(s)^2}|\{j\in \bar{F}-F^{(k)}:~|\bar{\beta}_j|\geq\gamma\}|,
  \end{align*}
  which implies that
  \begin{align*}
    &|\bar{F}-F^{(k)}|\geq 2|\{j\in \bar{F}-F^{(k)}:~|\bar{\beta}_j|\geq \gamma\}| = 2(|\bar{F}-F^{(k)}|-|\{j\in \bar{F}-F^{(k)}:~|\bar{\beta}_j|<\gamma\}|)\\
    \Rightarrow & |\bar{F}-F^{(k)}|\leq 2|\{j\in \bar{F}-F^{(k)}:~|\bar{\beta}_j|< \gamma\}|.
  \end{align*}
  The first inequality is obtained from Eq.~\eqref{eqn_thmmain_1} $$\|\bar\beta-\beta^{(k)}\|\leq {2\epsilon\over \rho_-(s)}\sqrt{|\bar{F}-F^{(k)}|}\leq {2\sqrt{2}\epsilon\over \rho_-(s)}\sqrt{|j\in \bar{F}-F^{(k)}: |\bar{\beta}_j|<\gamma|}.$$
  Next, we consider
  \begin{align*}
    &Q(\bar\beta)-Q(\beta^{(k)})\\
    \geq & \langle \triangledown Q(\beta^{(k)}), \bar\beta-\beta^{(k)} \rangle + {\rho_{-}(s) \over 2}\|\bar{\beta}-\beta^{(k)}\|^2\\
    = & \langle \triangledown Q(\beta^{(k)})_{\bar{F}-F^{(k)}}, (\bar\beta-\beta^{(k)})_{\bar{F}-F^{(k)}} \rangle + {\rho_{-}(s) \over 2}\|\bar{\beta}-\beta^{(k)}\|^2~~~~(\text{from~$\triangledown Q(\beta^{(k)})_{F^{(k)}}=0$})\\
    \geq & -\|\triangledown Q(\beta^{(k)})\|_\infty\|\bar{\beta}-\beta^{(k)}\|_1 + {\rho_{-}(s) \over 2}\|\bar{\beta}-\beta^{(k)}\|^2 \\
    \geq & -\epsilon\sqrt{|\bar{F}-F^{(k)}|}\|\bar{\beta}-\beta^{(k)}\| + {\rho_{-}(s) \over 2}\|\bar{\beta}-\beta^{(k)}\|^2\\
    \geq &{\rho_-(s)\over 2}\left(\|\bar\beta-\beta^{(k)}\| - \epsilon\sqrt{|\bar{F}-F^{(k)}|}/\rho_-(s)\right)^2 - \epsilon^2|\bar{F}-F^{(k)}|/(2\rho_-(s))\\
    \geq & - \epsilon^2|\bar{F}-F^{(k)}|/(2\rho_-(s))\\
    \geq & - {\epsilon^2\over \rho_-(s)}|\{j\in \bar{F}-F^{(k)}:~|\bar{\beta}_j|<\gamma\}|.
  \end{align*}
  It proves the second inequality.
The last inequality in this theorem is obtained from
  \begin{align*}
    &{\epsilon^2\over 2\rho_+(1)^2}|F^{(k)}-\bar{F}|\\
    \leq &\|(\beta^{(k)}-\bar\beta)_{F^{(k)}-\bar{F}}\|^2 ~~~~(\text{from Lemma~\ref{lem_gbc}})\\
    \leq &\|\beta^{(k)}-\bar\beta\|^2 \\
    \leq &{4\epsilon^2\over \rho_-(s)^2}|\bar{F}-F^{(k)}|\\
    \leq &{8\epsilon^2\over \rho_-(s)^2}|\{j\in \bar{F}-F^{(k)}: |\bar{\beta}_j|<\gamma\}|.
  \end{align*}
  It completes the proof.
\end{proof}

Next we will study the upper bound of ``k'' when Algorithm~\ref{alg_FoBaobj} terminates.
\begin{lemma}
(General forward step) Let $\beta = \arg\min_{supp(\beta)\subset F}
Q(\beta)$. For any $\beta'$ with support $F'$ and
$i\in\{i:~|\triangledown Q(\beta)_i| \geq \max_{j}|\triangledown
Q(\beta)_j|\}$, we have
$$|F'-F|(Q(\beta)-\min_{\eta}Q(\beta+\eta e_i))\geq {\rho_-(s)\over \rho_+(1)}(Q(\beta)-Q(\beta')).$$
\end{lemma}
\begin{proof}
The following proof follows the idea of Lemma A.3 in
\cite{Zhang11a}. Denote $i^*$ as $\arg\max_i|\triangledown
Q(\beta)_{i}|$. For all $j\in\{1,\cdots,d\}$, we define
$$P_j(\eta)=\eta \bold{sgn}(\beta'_j)\triangledown Q(\beta)_j +{\rho_+(1)\over 2}\eta^2.$$ It is easy to verify that $\min_{\eta}P_j(\eta) = -{|\triangledown Q(\beta)_j|^2\over 2\rho_+(1)}$ and
\begin{equation}
\min_\eta P_i(\eta) \leq \min_{\eta} P_{i^*}(\eta)=\min_{j,\eta}
P_{j}(\eta)\leq \min_{j}P_j(\eta). \label{eqn_lemgfs_1}
\end{equation}

Denoting $u=\|\beta'_{F'-F}\|_1$, we obtain that
\begin{align*}
  u\min_{j}P_j(\eta) \leq & \sum_{j\in F'-F}|\beta'_j|P_j(\eta) \\
  = & \eta\sum_{j\in F'-F}\beta'_j\triangledown Q(\beta)_j + {u\rho_+(1)\over 2}\eta^2 \\
  = & \eta\langle\beta'_{F'-F}, \triangledown Q(\beta)_{F'-F}\rangle + {u\rho_+(1)\over 2}\eta^2 \\
  = & \eta\langle\beta'-\beta, \triangledown Q(\beta)\rangle + {u\rho_+(1)\over 2}\eta^2~~~~(\text{from $\triangledown Q(\beta)_{F}=0$}) \\
  \leq &\eta(Q(\beta')-Q(\beta)-{\rho_-(s)\over 2}\|\beta'-\beta\|^2) + {u\rho_+(1)\over 2}\eta^2.
 \end{align*}
 Taking $\eta = (Q(\beta)-Q(\beta')+{\rho_-(s)\over 2}\|\beta'-\beta\|^2)/{u\rho_+(1)}$ on both sides, we get
 \begin{align*}
 u\min_{j}P_j(\eta) \leq -{\left(Q(\beta)-Q(\beta') + {\rho_-(s)\over 2}\|\beta-\beta'\|^2\right)^2 \over 2\rho_+(1)u}\leq -{(Q(\beta)-Q(\beta'))\rho_-(s)\|\beta-\beta'\|^2\over \rho_+(1)u},
 \end{align*}
 where it uses the inequality $(a+b)^2\geq 4ab$. From Eq.~\eqref{eqn_lemgfs_1}, we have
 \begin{align*}
   \min_\eta P_i(\eta)\leq & \min_{j}P_j(\eta) \\
   \leq &-{(Q(\beta)-Q(\beta'))\rho_-(s)\|\beta-\beta'\|^2\over \rho_+(1)u^2} \\
   \leq &-{\rho_-(s)\over \rho_+(1)|F'-F|}(Q(\beta)-Q(\beta')),
 \end{align*}
 where we use ${\|\beta-\beta'\|^2\over u^2} = {\|\beta-\beta'\|^2\over \|(\beta-\beta')_{F'-F}\|^2} \geq {1\over |F'-F|}$. Finally, we use $$\min_\eta Q(\beta+\eta e_i) - Q(\beta) \leq \min_\eta Q(\beta+\eta \bold{sgn}(\beta'_i)e_i) - Q(\beta) \leq \min_\eta P_i(\eta)$$ to prove the claim.
\end{proof}

\noindent {\bf Proof of Theorem~\ref{thm_main2}}
\begin{proof}
  Please refer to the proof of Theorem~\ref{thm_main2_obj}.
\end{proof}

\section{Proofs of RSCC}

\noindent{\bf Proof of Theorem~\ref{thm_kappa}}
\begin{proof}
First from the random matrix theory~\citep{Vershynin11}, we have that for any random matrix $X\in\mathbb{R}^{n\times d}$ with independent sub-gaussian isotropic random rows or columns, there exists a constant $C_0$ such that
\begin{equation}
\sqrt{n} - C_0 \sqrt{s\log d} \leq \min_{\|h\|_0\leq s} \frac{\|Xh\|}{\|h\|} \leq \max_{\|h\|_0\leq s} \frac{\|Xh\|}{\|h\|} \leq \sqrt{n} + C_0\sqrt{s\log d}
\label{eqn_ripbound}
\end{equation}
holds with high probability.

From the mean value theorem, for any $\beta', \beta \in \D_s$, there exists a point $\theta$ in the segment of $\beta'$ and $\beta$ satisfying
\begin{align}
&
Q(\beta')-Q(\beta) - \langle \nabla Q(\beta), \alpha-\beta \rangle = {1\over 2}(\beta'-\beta)^T\nabla^2 Q(\theta)(\beta'-\beta)
\nonumber
\\ & \quad
={1\over 2}(\beta'-\beta)^T\left({1\over n}\sum_{i=1}^nX^T_{i.}\nabla^2_1l_i(X_{i.}\theta, y_i)X_{i.} + \nabla^2 R(\theta) \right)(\beta'-\beta)
\nonumber
\\ & \quad\quad
\text{(from $\theta\in\D_s$)}
\nonumber
\\ & \quad
\geq {1\over 2}(\beta'-\beta)^T\left({1\over n}\sum_{i=1}^n\lambda^-X^T_{i.}X_{i.} + \lambda^-_R I \right)(\beta'-\beta)
\nonumber
\\ & \quad
= {1\over 2}(\beta'-\beta)^T\left({\lambda^-\over n}X^TX + \lambda^-_R I \right)(\beta'-\beta)
\nonumber
\\ & \quad
= {\lambda^-\over 2n}\|X(\beta'-\beta)\|^2 + {\lambda^-_R\over 2} \|\beta'-\beta\|^2
\label{eqn_proof_kappa_1}
\end{align}
Similarly, one can easily obtain
\begin{equation}
Q(\beta')-Q(\beta) - \langle \nabla Q(\beta), \alpha-\beta \rangle \leq {\lambda^+\over 2n}\|X(\beta'-\beta)\|^2 + {\lambda^+_R\over 2} \|\beta'-\beta\|^2
\label{eqn_proof_kappa_2}
\end{equation}
Using~\eqref{eqn_ripbound}, we have
\[
(\sqrt{n} - C_0 \sqrt{s\log d})^2 \|\beta'-\beta\|^2 \leq \|X(\beta'-\beta)\|^2 \leq (\sqrt{n} + C_0 \sqrt{s\log d})^2 \|\beta'-\beta\|^2
\]
Together with~\eqref{eqn_proof_kappa_1} and \eqref{eqn_proof_kappa_2}, we obtain
\begin{align*}
 &{1\over 2}\left({\lambda^-_R + \lambda^-\left(1-C_0\sqrt{\frac{s\log d}{n}}\right)^2}\right)\|\beta'-\beta\|^2 \leq Q(\beta')-Q(\beta) - \langle \nabla Q(\beta), \alpha-\beta \rangle \\
 &\quad \leq {1\over 2}\left({\lambda^+_R + \lambda^+\left(1+C_0\sqrt{\frac{s\log d}{n}}\right)^2}\right)\|\beta'-\beta\|^2,
\end{align*}
which implies the claims \eqref{eqn_thm_kappa_+}, \eqref{eqn_thm_kappa_+}, and \eqref{eqn_thm_kappa_+-} by taking $n \geq Cs\log d$ where $C = C_0^2 / (1-\sqrt{2}/2)^{2}$.

Next we apply this result to prove the rest part. In Algorithm~\ref{alg_FoBaobj}, since choose $n\geq Cs\log d$, we have that $\rho_-(s)$, $\rho_+(s)$, and $\kappa(s)$ in Algorithm~\ref{alg_FoBaobj} with FoBa-obj satisfy the bounds in \eqref{eqn_thm_kappa_all}. To see why $s$ chosen in~\eqref{eqn_thm_kappa_s} satisfies the condition in~\eqref{eq:thm_1}, we verify the right hand side of \eqref{eq:thm_1}:
\[
RHS=(\bar{k}+1)\left[\left(\sqrt{\rho_+(s)\over \rho_-(s)}+1\right){2\rho_+(1)\over \rho_-(s)}\right]^2 \leq 4(\bar{k}+1)(\sqrt{\kappa}+1)^2\kappa^2 \leq s - \bar{k}.
\]
Therefore, Algorithm~\ref{alg_FoBaobj} terminates at most $4\kappa^2(\sqrt{\kappa}+1)^2(\bar{k}+1)$ iterations with high probability from Theorem~\ref{thm_main2_obj}. The claims for Algorithm~\ref{alg_FoBaobj} with FoBa-gdt can be proven similarly.

\end{proof}

\clearpage
\begin{center}
{\huge\bf
Appendix: Additional Experiments
}
\end{center}
\section{Additional Experiments} \label{sec:addexp}

We conduct additional experiments using two models: logistic regression and linear-chain CRFs. Four algorithms (FoBa-gdt, FoBa-obj, and their corresponding forward
greedy algorithms, i.e., Forward-obj and Forward-gdt) are compared on both synthetic data and real world data. Note that comparisons of FoBa-obj and L1-regularization methods can be found in previous studies~\citep{Jalali11, Zhang11}.

It is hard to evaluate $\rho_-(s)$ and $\rho_+(s)$ in practice, and we typically employ the following strategies to set $\delta$ in FoBa-obj and FoBa-gdt:
\begin{itemize}
\item If we have the knowledge of $\bar{F}$ (e.g., in synthetic simulations), we can compute the true solution $\bar{\beta}$. $\delta$ and $\epsilon$ are given as $\delta = Q(\bar{\beta}) - \min_{\alpha, j\notin \bar{F}} Q(\bar\beta + \alpha e_j)$ and $\epsilon = \|\triangledown Q(\bar\beta)\|_\infty$.
\item We can change the stopping criteria in both FoBa-obj and FoBa-gdt based on the sparsity of $\beta^{(k)}$, that is, when the number of nonzero entries in $\beta^{(k)}$ exceeds a given sparsity level, the algorithm terminates.
\end{itemize}
We employed the first strategy only in synthetic simulations for logistic regression and the second strategy in all other cases.


\subsection{Experiments on Logistic Regression}


L$_2$-regularized logistic regression is a typical example of models
with smooth convex loss functions:
$$Q(\beta) = {1\over n}\sum_i \log (1+\exp(-y_i X_{i.}\beta)) + {1\over 2}\lambda ||\beta||^2.$$
where $X_{i.}\in\mathbb{R}^{d}$, that is, the $i^{th}$ row of $X$, is the $i^{th}$ sample and $y_i\in\{1,-1\}$ is the corresponding label. This is a binary
classification problem for finding a hyperplane defined by
$\beta$ in order to separate two classes of data in $X$. $\log(1+\exp (.))$
defines the penalty for misclassification. The $\ell_2$-norm is used for
regularization in order to mitigate overfitting. The gradient of $Q(\beta)$
is computed by
\begin{align*}
\frac{\partial Q(\beta)}{\partial \beta} = {1\over n}\sum_{i} \frac{-y_i\exp(-y_iX_{i.}\beta)}{1+\exp (-y_iX_{i.}\beta)}X^T_{i.} + \lambda \beta.
\end{align*}

\subsubsection{Synthetic Data}

We first compare FoBa-obj and FoBa-gdt with Forward-obj and
Forward-gdt\footnote{Forward-gdt is equivalent to the orthogonal
matching pursuit in \citep{Zhang09}.} using synthetic data generated as
follows. {\rt The data matrix $X \in \mathbb{R}^{n \times
d}$ consistes of two classes of data (each row is a sample) from
i.i.d. Gaussian distribution $\mathcal{N}(\beta^*, I_{d \times d})$
and $\mathcal{N}(-\beta^*, I_{d \times d})$, respectively, where the
true parameter $\beta^*$ is sparse and its nonzero entries follow
i.i.d. uniform distribution $\mathcal{U}[0,1]$. The two classes are
of the same size.} The norm of $\beta^*$ is normalized to 5 and its
sparseness ranges from 5 to 14. The vector $y$ is the class label
vector with values being 1 or -1. We set $\lambda=0.01, n=100$, and
$d=500$. The results are averaged over 50 random trials.

Four evaluation measures are considered in our comparison:
F-measure ($2|F^{(k)} \cap \bar{F}| / (|F^{(k)}| + |\bar{F}|)$),
estimation error ($\|\beta^{(k)} - \bar{\beta}\| /
\|\bar{\beta}\|$), objective value, and the number of nonzero
elements in $\beta^{(k)}$ (the value of $k$ when the algorithm
terminates). The results in Fig.~\ref{fig_LGT} suggest that:
\begin{itemize}
\item The FoBa algorithms outperform the forward algorithms overall, especially in F-measure (feature selection) performance;
\item The FoBa algorithms are likely to yield sparser solutions than the forward algorithms, since they allow for the removal of bad features;
\item There are minor differences in estimation errors and objective values between the FoBa algorithms and the forward algorithms;
\item FoBa-obj and FoBa-gdt perform similarly. The former gives better objective values and the latter is better in estimation error;
\item FoBa-obj and Forward-obj decrease objective values more quickly than FoBa-gdt and Forward-gdt since they employ the direct ``goodness'' measure (reduction of the objective value).
\end{itemize}

\begin{figure*}[ti]
  \centering
    \subfigure{\includegraphics[scale=0.22]{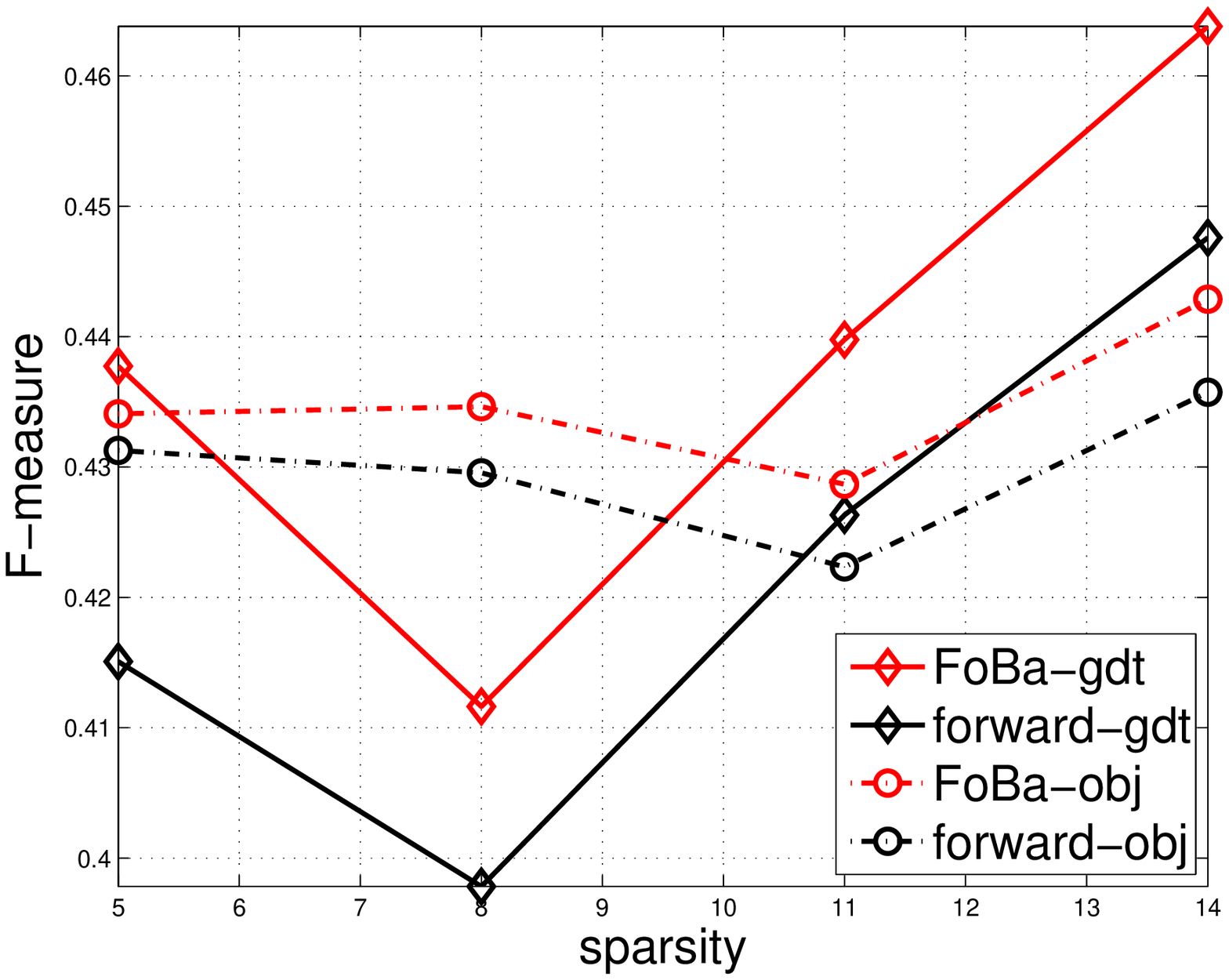}}
    \subfigure{\includegraphics[scale=0.22]{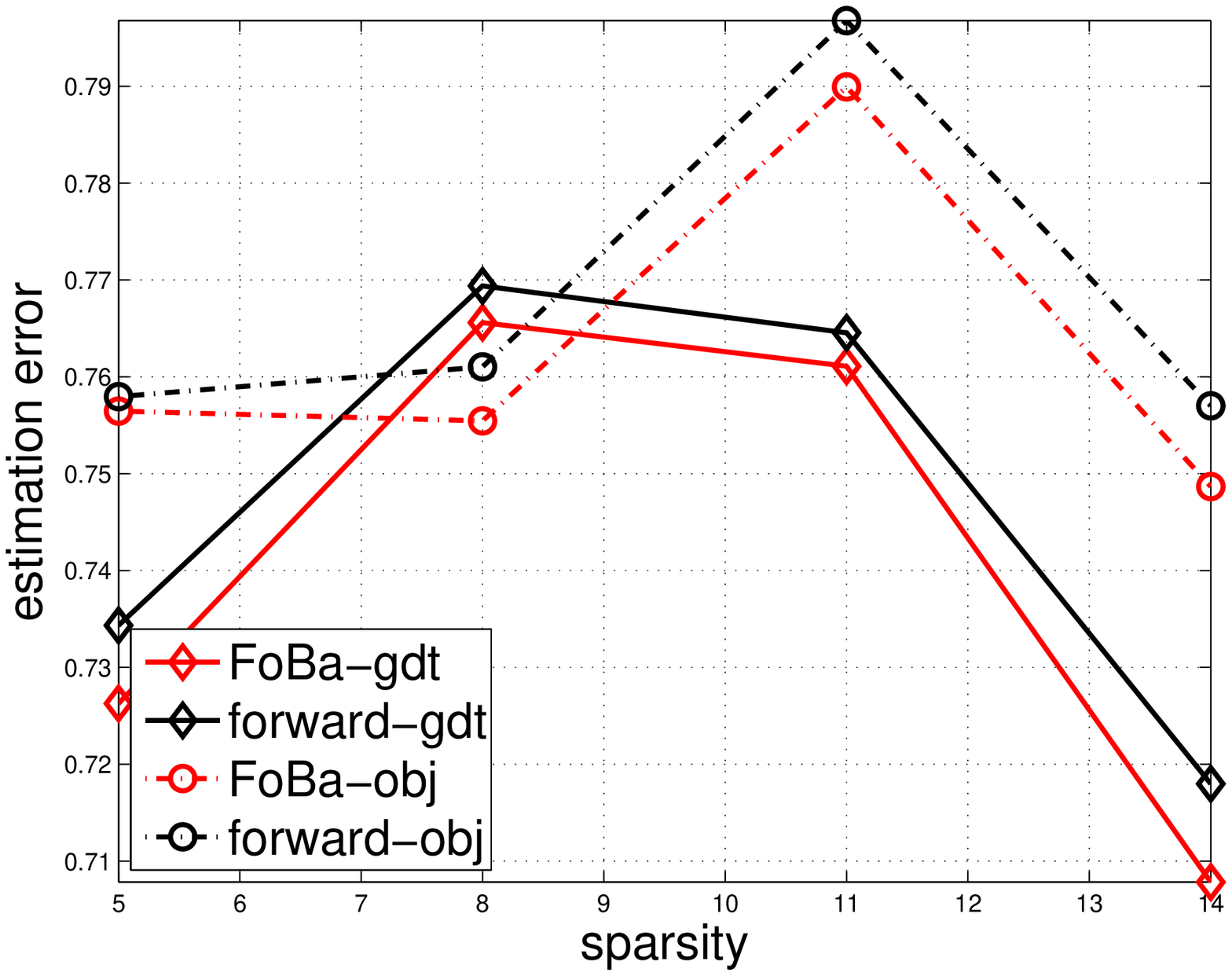}}
    \subfigure{\includegraphics[scale=0.22]{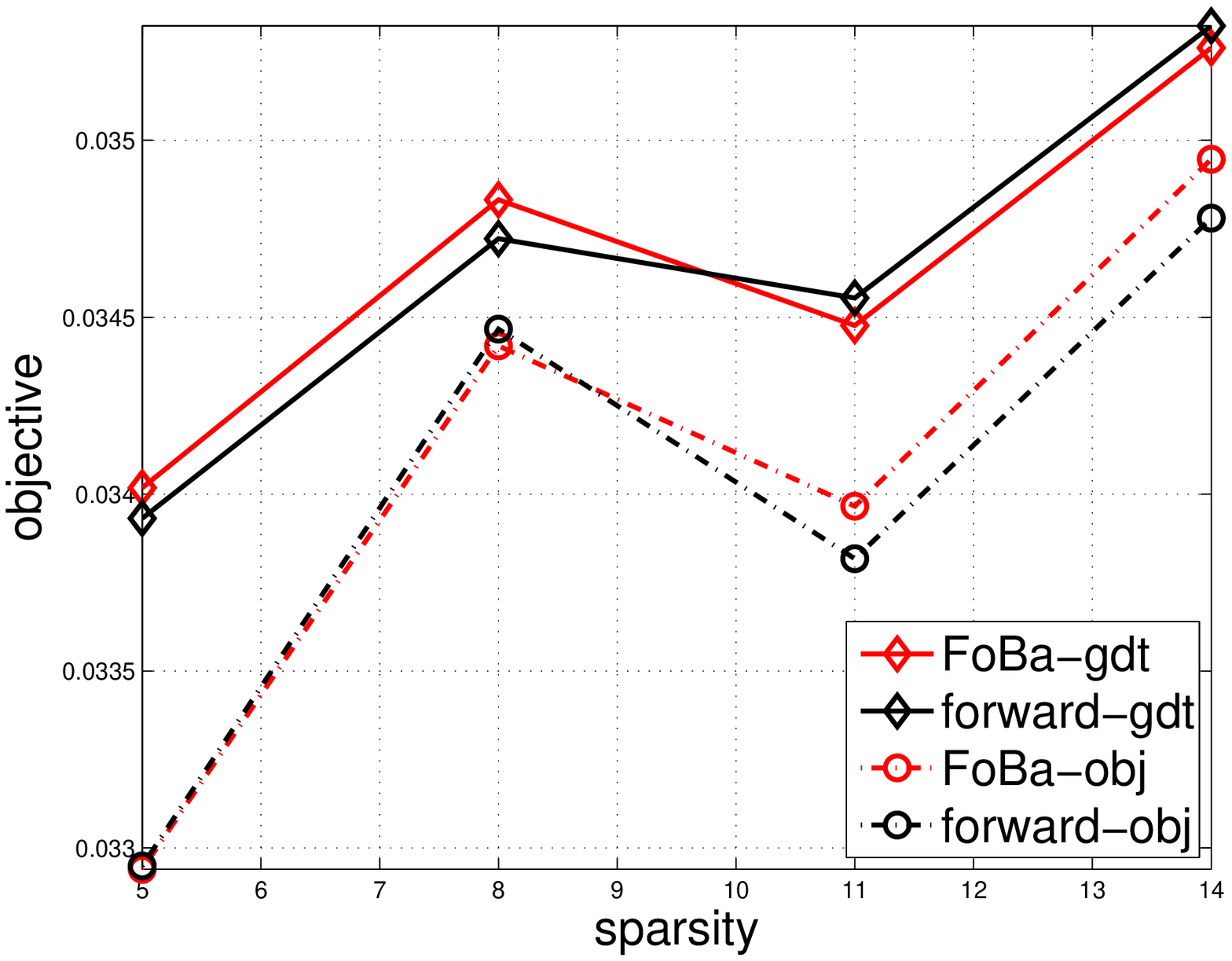}}
    \subfigure{\includegraphics[scale=0.22]{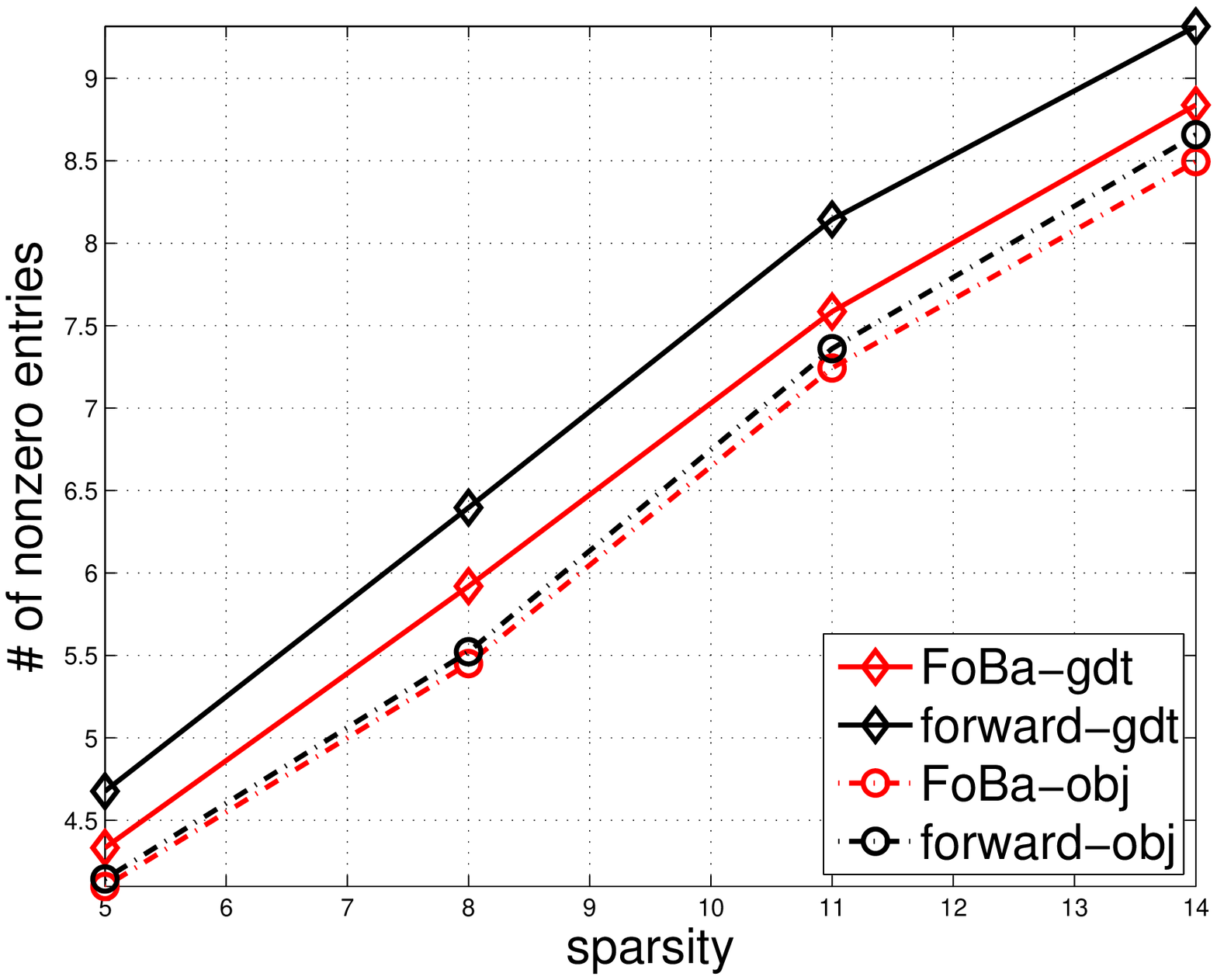}}
    \caption{Comparison for F-measure, estimation error, objective value, and the number of nonzero elements. The horizontal axis is the sparsity number of the true model ranging from 5 to 14.}
    \label{fig_LGT}
\end{figure*}

As shown in Table~\ref{tab_ct_syn}, FoBa-gdt and Forward-gdt are much more efficient than FoBa-obj and Forward-obj,
because FoBa-obj and Forward-obj solve a large number of optimization problems in each iteration.

\begin{table}[t]
\centering
{\bt \caption{Computation time (seconds) on synthetic data (top:
logistic regression; bottom: CRF).}
  \begin{tabular}{ c  c  c  c c}
\\ \toprule
    $S$ &FoBa-gdt&Forward-gdt&FoBa-obj&Forward-obj\\ \midrule
    5 &1.96e-2&0.65e-2&1.25e+1&1.24e+1 \\
    8 &1.16e-2&0.76e-2&1.52e+1&1.52e+1 \\
    11 &1.53e-2&1.05e-2&1.83e+1&1.80e+1 \\
    14 &1.51e-2&1.00e-2&2.29e+1&2.25e+1 \\ \midrule
   10&1.02e+0  &  0.52e+0 &  3.95e+1  & 1.88e+1 \\
    15 & 1.85e+0  &  1.05e+0  & 5.53e+1 &  2.65e+1 \\
    20 & 2.51e+0  &  1.53e+0 &  6.89e+1  & 3.33e+1 \\
    25 & 3.86e+0  &  2.00e+0 &  8.06e+1 &  3.91e+1 \\
    30 & 5.13e+0  &  3.00e+0 &  9.04e+1  & 4.39e+1 \\
    \bottomrule
  \end{tabular}\label{tab_ct_syn}
}
\end{table}

\subsubsection{UCI data}
\begin{table}[h]
\caption{Comparison of different algorithms in terms of training error, testing error, and objective function value (from left to right) on a1a, a2a, and w1a data sets (from top to bottom). The values of $n$ denotes the number of training/testing samples.}
\centering
  \begin{tabular}{ c  c  c  c }
\toprule
$S$ &objective&training error&testing error \\
\midrule \multicolumn{4}{c} {a1a (FoBa-obj/FoBa-gdt), $n=1605/30956,
d=123$}\\ \midrule
    10&{\bf556}/558&{\bf0.169}/0.17&0.165/{\bf0.164}\\
    20&{\bf517}/529&0.163/{\bf0.16}&0.163/{\bf0.160}\\
    30&{\bf499}/515&0.151/{\bf0.145}&0.162/{\bf0.158}\\
    40&{\bf490}/509&0.150/{\bf0.146}&0.162/{\bf0.159}\\
    50&{\bf484}/494&0.146/{\bf0.141}&0.162/{\bf0.161}\\
    60&{\bf481}/484&{\bf0.14}/0.141&0.166/{\bf0.161}\\
    70&{\bf480}/{\bf480}&{\bf0.138}/0.139&0.166/{\bf0.163}\\ \midrule
\multicolumn{4}{c} {a2a (FoBa-obj/FoBa-gdt), $n=2205/30296,
d=123$}\\ \midrule
    10&870/{\bf836}&0.191/{\bf0.185}&0.174/{\bf0.160}\\
    20&{\bf801}/810&0.18/{\bf0.177}&0.163/{\bf0.157}\\
    30&{\bf775}/790&0.172/{\bf0.168}&0.165/{\bf0.156}\\
    40&{\bf758}/776&{\bf0.162}/0.167&0.163/{\bf0.155}\\
    50&{\bf749}/764&{\bf0.163}/0.166&0.162/{\bf0.157}\\
    60&{\bf743}/750&0.162/{\bf0.161}&0.163/{\bf0.160}\\
    70&{\bf740}/742&0.162/{\bf0.158}&0.163/{\bf0.161}\\ \midrule
\multicolumn{4}{c} {w1a (FoBa-obj/FoBa-gdt), $n=2477/47272, d=300$}
\\ \midrule
    10&{\bf574}/595&0.135/{\bf0.125}&0.142/{\bf0.127}\\
    20&{\bf424}/487&0.0969/{\bf0.0959}&0.11/{\bf0.104}\\
    30&{\bf341}/395&0.0813/{\bf0.0805}&0.0993/{\bf0.0923}\\
    40&{\bf288}/337&{\bf0.0704}/0.0715&0.089/{\bf0.0853}\\
    50&{\bf238}/282&{\bf0.06}/0.0658&0.0832/{\bf0.0825}\\
    60&{\bf215}/226&{\bf0.0547}/0.0553&{\bf0.0814}/0.0818\\
    70&{\bf198}/206&{\bf0.0511}/{\bf0.0511}&0.0776/{\bf0.0775}\\\bottomrule
  \end{tabular}
\label{tab_lgt_real}
\end{table}

We next compare FoBa-obj and FoBa-gdt on UCI data sets. We use
$\lambda=10^{-4}$ in all experiments. Note that the sparsity numbers here
are different from those in the past experiment. Here
``S'' denotes the number of nonzero entries in the output. We
use three evaluation measures: 1) training classification error, 2)
test classification error, and 3) training objective value (i.e.,
the value of $Q(\beta^{(k)})$). The datasets are available
online\footnote{\url{http://www.csie.ntu.edu.tw/~cjlin/libsvmtools/
datasets}}. From Table~\ref{tab_lgt_real}, we observe that
\begin{itemize}
\item FoBa-obj and FoBa-gdt obtain similar training errors in all cases;
\item FoBa-obj usually obtain a lower objective value than FoBa-gdt, but this does not imply that FoBa-obj selects better features. Indeed, FoBa-gdt achieves lower test errors in several cases.
\end{itemize}
We also compare the computation time FoBa-obj and FoBa-gdt and
found that FoBa-gdt was empirically at least 10 times faster than
FoBa-obj (detailed results omitted due to space limitations).
In summary, FoBa-gdt has the same theoretical guarantee as FoBa-obj, and empirically FoBa-dgt achieves competitive in terms of empirical performance; in terms of computational time, FoBa-gdt is much more efficient.

\subsection{Simulations on Linear-Chain CRFs}\label{sec:CRFsyn}
Another typical class of model family with smooth convex loss
functions is conditional random fields~(CRFs) \citep{LaffertyMP01, Sutton06}, which is most commonly used in segmentation and tagging
problems with respect to sequential data. Let
$X_t\in\mathbb{R}^{D},~t=1,2,\cdots, T$ be a sequence of random
vectors. The corresponding labels of $X_t$'s are stored in a vector
$y\in \mathbb{R}^{T}$. Let $\beta\in\mathbb{R}^{M}$ be a parameter
vector and $f_m(y, y', x_t),~m=1,\cdots,M$, be a set of real-valued
feature functions. A linear-chain CRFs will then be a distribution
$\mathbb{P}_\beta(y|X)$ that takes the form
\begin{equation*}
\mathbb{P}_\beta(y|X) = {1\over Z(X)}\prod_{t=1}^T
\exp\left\{\sum_{m=1}^M\beta_mf_m(y_t, y_{t-1}, X_t)\right\}
\end{equation*}
where $Z(X)$ is an instance-specific normalization function
\begin{equation}
Z(X) = \sum_y\prod_{t=1}^T \exp\left\{\sum_{m=1}^M\beta_mf_m(y_t,
y_{t-1}, X_t)\right\}. \label{eqn_CRF_Z}
\end{equation}
In order to simplify the introduction below, we assume that all
$X_t$'s are discrete random vectors with limited states. If we have
training data $X^{i}$ and $y^i$($i=1,\cdots, I$) we can estimate the
value of the parameter $\beta$ by maximizing the likelihood
$\prod_{i}\mathbb{P}_{\beta}(y^i|X^i)$, which is equivalent to
minimizing the negative log-likelihood:
\begin{align*}
\min_{\beta}:~&Q(\beta):=-\log\left(\prod_{i=1}^I\mathbb{P}_\beta(y^i|X^i)\right)\\
=&-\sum_{i=1}^I\sum_{t=1}^T\sum_{m=1}^M\beta_mf_m(y_t^i,y_{t-1}^i,x_t^i)+\sum_{i=1}^I\log
Z(X^i).
\end{align*}
The gradient is computed by
\begin{align*}
{\partial Q(\beta)\over \partial
\beta_m}=&-\sum_{i}\sum_{t}f_m(y^i_t, y^i_{t-1},
X^{i}_t)+\sum_{i}{1\over Z(X^i)}{\partial Z(X^i)\over \partial \beta_{m}}
\end{align*}
Note that it would be extremely expensive (the complexity would be
$O(2^T)$) to directly calculate $Z(X^i)$ from
\eqref{eqn_CRF_Z}, since the sample space of $y$ is too large. So is
${\partial Z(X^i)\over \partial \beta_{m}}$. The sum-product
algorithm can reduce complexity to polynomial time in terms
of $T$(see \citep{Sutton06} for details).

{\bt In this experiment, we use a data generator provided in Matlab
CRF\footnote{\url{http://www.di.ens.fr/~mschmidt/Software/crfChain.html}}.
We omit a detailed description of the data generator here because of
space limitations. We generate a data matrix $X\in\{1,2,3,4,5\}^{T\times D}$
with $T=800$ and $D=4$ and the corresponding label vector
$y\in\{1,2,3,4\}^{T}$~(four classes)\footnote{We here choose a
small sample size because of the high computational costs of FoBa-obj
and Forward-obj.}. Because of our target application, i.e., sensor
selection, we only enforced sparseness to the features related to
observation (56 in all) and the features w.r.t. transitions
remained dense.
Fig.~\ref{fig_CRF_syn} shows the performance, averaged over 20 runs,
with respect to training error, test error, and objective value. We
observe from Fig~\ref{fig_CRF_syn}:}
\begin{itemize}
\item FoBa-obj and Forward-obj gives lower objective values, training error, and test error than FoBa-gdt and Forward-gdt when the sparsity level is low.
\item The performance gap between different algorithms becomes smaller when the sparsity level increases.
\item The performance of FoBa-obj is quite close to that of Forward-obj, while FoBa-gdt outperforms Forward-gdt particularly when the sparsity level was low.
\end{itemize}

{\lt The bottom of Table~\ref{tab_ct_syn} shows computation
time for all four algorithms. Similar to the logistic
regression case, the gradient based algorithms (FoBa-gdt and
Forward-gdt) are much more efficient than the objective based
algorithms FoBa-obj and Forward-obj. From Table~\ref{tab_ct_syn}, FoBa-obj is computationally more expensive than Forward-obj, while they are comparable in the logistic regression case.} The main reason for this is that the
evaluation of objective values in the CRF case is more expensive than that in the logistic regression case. These results demonstrate the potential of FoBa-gdt.

\begin{figure*}[h!]
  \centering
    \subfigure{\includegraphics[scale=0.30]{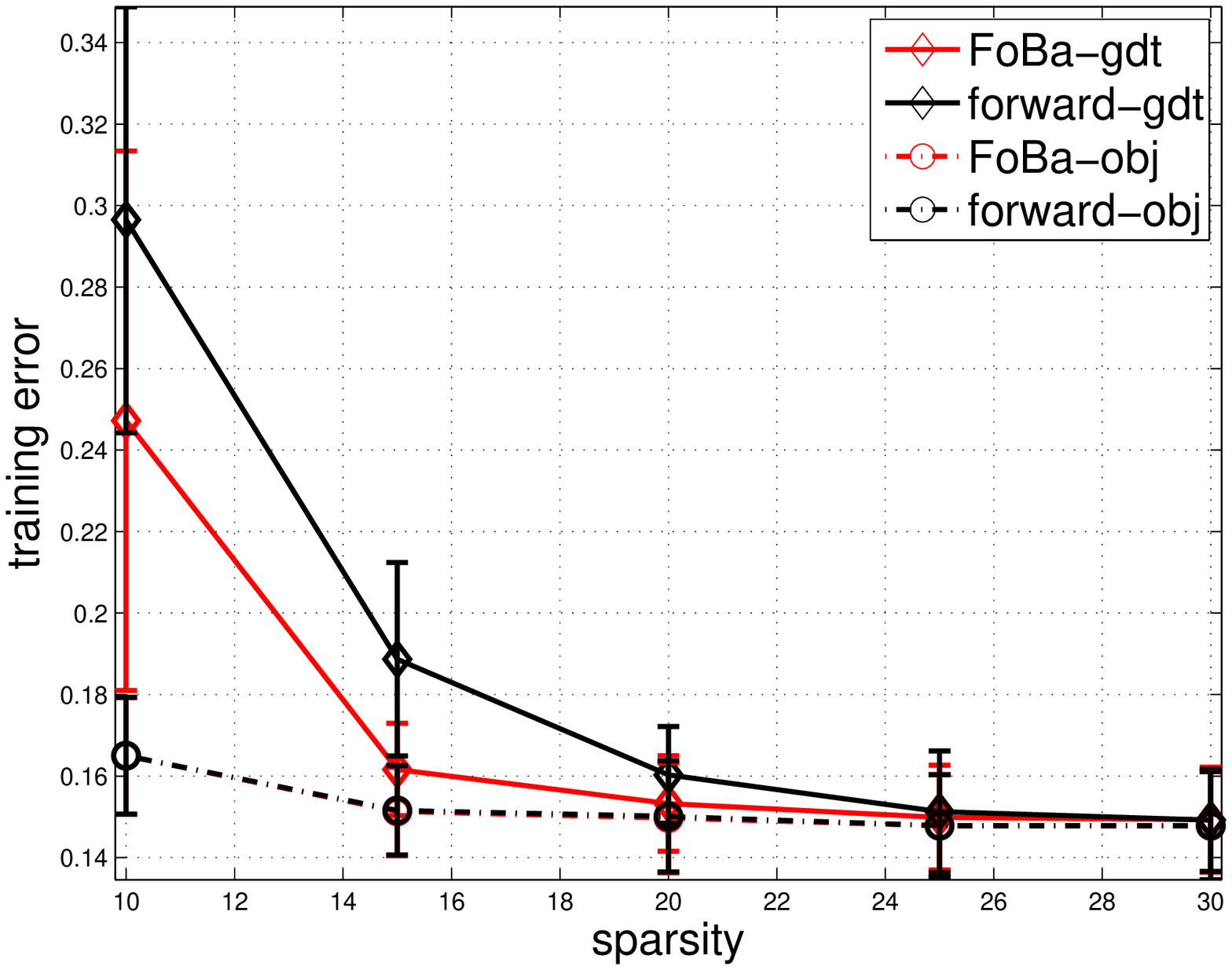}}
        \subfigure{\includegraphics[scale=0.30]{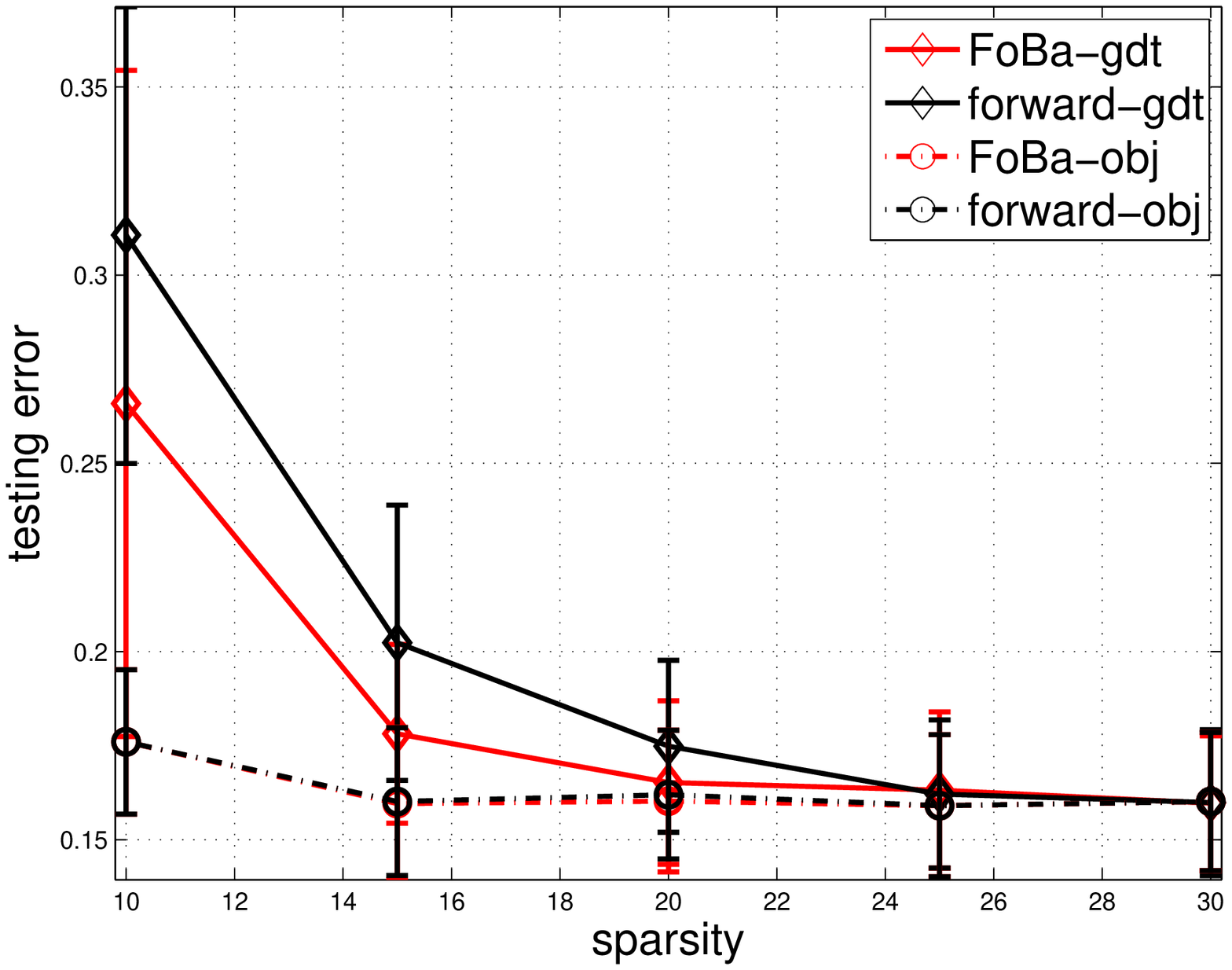}}
        \subfigure{\includegraphics[scale=0.30]{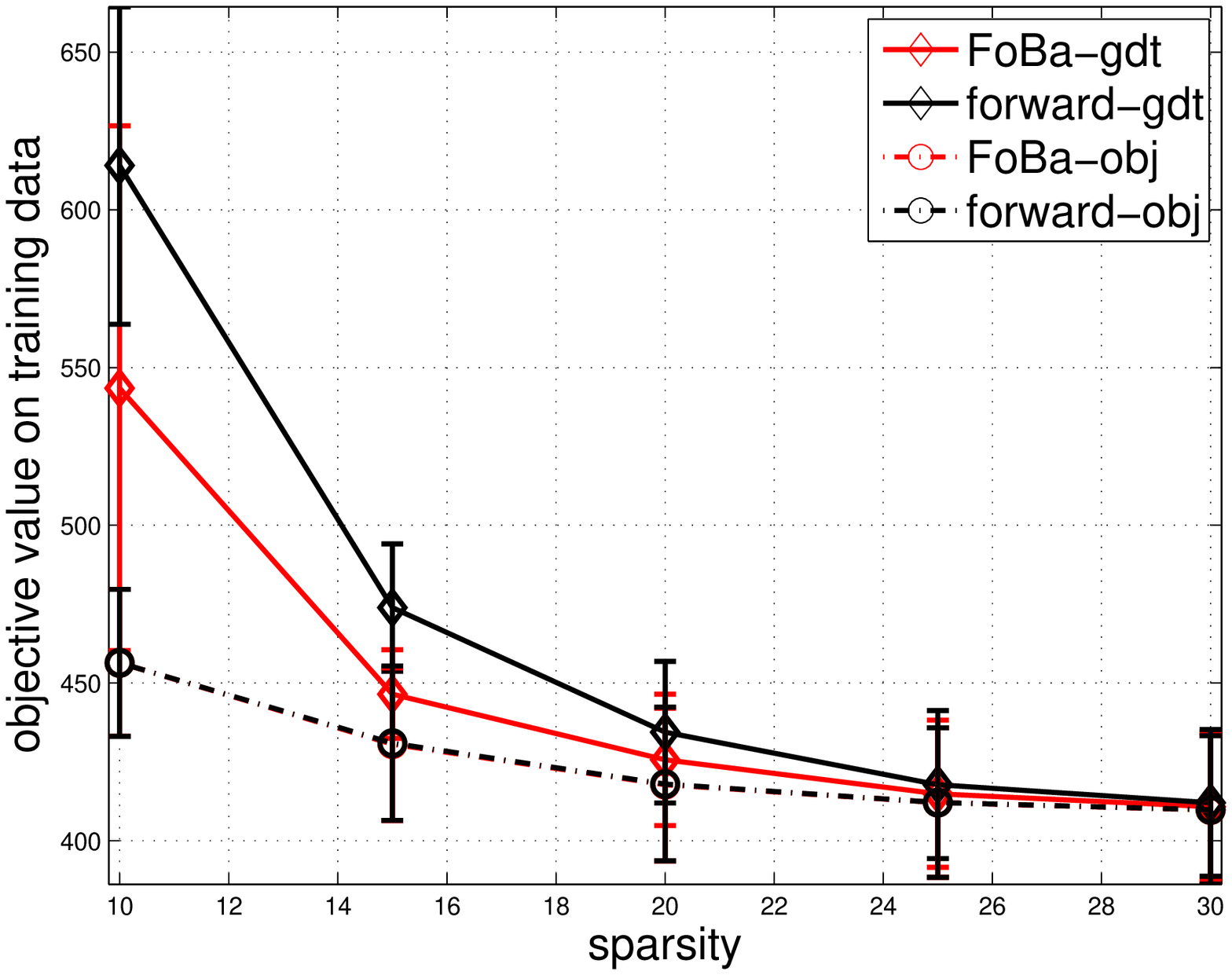}}
    \caption{Comparison of different algorithms in terms of training error and test error on synthetic data for CRF. The horizontal axis is the sparsity number of the true model ranging from 10 to 30.}
    \label{fig_CRF_syn}
\end{figure*}

}
\end{document}